\documentclass[11pt]{article}

\usepackage{algorithm}
\usepackage{scalerel}
\usepackage{mathrsfs}
\usepackage{enumitem}
\usepackage{bbm}
\usepackage{tikz}
\usepackage{float}
\usetikzlibrary{backgrounds}
\usepackage{color}
\usepackage{graphicx}
\usepackage{latexsym}
\usepackage{amsfonts}
\usepackage{pifont,xspace,fullpage,epsfig, wrapfig}
\usepackage{amsmath, amssymb, amsthm} %
\usepackage{multirow}
\usepackage{dsfont}
\usepackage{array}
\usepackage{adjustbox}

\usepackage{bookmark}

\usepackage{float}

\DeclareSymbolFont{AMSb}{U}{msb}{m}{n}
\DeclareMathSymbol{\N}{\mathbin}{AMSb}{"4E}
\DeclareMathSymbol{\Z}{\mathbin}{AMSb}{"5A}
\DeclareMathSymbol{\R}{\mathbin}{AMSb}{"52}
\DeclareMathSymbol{\Q}{\mathbin}{AMSb}{"51}
\DeclareMathSymbol{\erert}{\mathbin}{AMSb}{"50}
\DeclareMathSymbol{\I}{\mathbin}{AMSb}{"49}
\DeclareMathSymbol{\C}{\mathbin}{AMSb}{"43}

\definecolor{gray}{gray}{0.4}

\newcommand{\remove}[1]{}

\newtheorem{theorem}{Theorem}[section]
\newtheorem{lemma}[theorem]{Lemma}
\newtheorem{definition}[theorem]{Definition}
\newtheorem{remark}[theorem]{Remark}
\newtheorem{notation}[theorem]{Notation}

\newtheorem{claim}[theorem]{Claim}

\newtheorem{example}[theorem]{Example}

\newcommand{\1}{\mathbbm{1}}

\newcommand{\AAA}{\mathcal A}
\newcommand{\BB}{\mathcal B}
\newcommand{\BBB}{\mathcal B}

\newcommand{\DDD}{\mathcal D}

\newcommand{\eps}{\varepsilon}

\newcommand{\Lap}{\operatorname{\rm Lap}}
\newcommand{\Noise}{\operatorname{\rm Noise}}

\def\E{\operatorname*{\mathbb{E}}}

\def\Q{\operatorname*{\mathbb{Q}}}

\def\Lap{\mathop{\rm{Lap}}\nolimits}

\makeatletter
\newlength{\fboxhsep}
\newlength{\fboxvsep}
\newlength{\fboxtoprule}
\newlength{\fboxbottomrule}
\newlength{\fboxleftrule}
\newlength{\fboxrightrule}
\def\@frameb@xother#1{%
  \@tempdima\fboxtoprule
  \advance\@tempdima\fboxvsep
  \advance\@tempdima\dp\@tempboxa
  \hbox{%
    \lower\@tempdima\hbox{%
      \vbox{%
        \hrule\@height\fboxtoprule
        \hbox{%
          \vrule\@width\fboxleftrule
          #1%
          \vbox{%
            \vskip\fboxvsep
            \box\@tempboxa
            \vskip\fboxvsep}%
          #1%
          \vrule\@width\fboxrightrule}%
        \hrule\@height\fboxbottomrule}%
    }%
  }%
}
\long\def\fboxother#1{%
  \leavevmode
  \setbox\@tempboxa\hbox{%
    \color@begingroup
    \kern\fboxhsep{#1}\kern\fboxhsep
    \color@endgroup}%
  \@frameb@xother\relax}

\makeatother

\newcommand{\adaptivealg}[2]{\texttt{Adaptive}_{#1 {\rightleftarrows} #2}}

\begin{document}

\begin{titlepage}

\title{The Sparse Vector Technique, Revisited}
\author{
Haim Kaplan\thanks{Tel Aviv University and Google Research. \texttt{haimk@tau.ac.il}. Partially supported by the Israel Science Foundation (grant 1595/19), the German-Israeli Foundation (grant 1367/2017), and by the Blavatnik Family Foundation.}
\and
Yishay Mansour\thanks{Tel Aviv University and Google Research. \texttt{mansour.yishay@gmail.com}. This project has received funding from the European Research Council (ERC) under the European Union's Horizon 2020 research and innovation program (grant agreement No. 882396), and by the Israel Science Foundation (grant number 993/17).}
\and
Uri Stemmer\thanks{Ben-Gurion University and Google Research. \texttt{u@uri.co.il}. Partially Supported by the Israel Science Foundation (grant 1871/19) and by the Cyber Security Research Center at Ben-Gurion University of the Negev.}
}

\date{November 16, 2020}
\maketitle
\setcounter{page}{0} \thispagestyle{empty}

\begin{abstract}
We revisit one of the most basic and widely applicable techniques in the literature of differential privacy -- the {\em sparse vector technique} [Dwork et al.,\ STOC 2009]. This simple algorithm privately tests whether the value of a given query on a database is close to what we expect it to be. It allows to ask an unbounded number of queries as long as the answer is close to what we expect, and halts following the first query for which this is not the case.

We suggest an alternative, equally simple, algorithm that can continue testing queries as long as any single individual does not contribute to the answer of too many queries whose answer deviates substantially form what we expect. Our analysis is subtle and some of its ingredients may be more widely applicable.
In some cases our new algorithm allows to privately extract much more information from the database than the original.

We demonstrate this by applying our algorithm to the {\em shifting heavy-hitters} problem: On every time step, each of $n$ users gets a new input, and the task is to privately identify all the {\em current} heavy-hitters. That is, on time step $i$, the goal is to identify all data elements $x$ such that many of the users have $x$ as their current input. We present an algorithm for this problem with improved error guarantees over what can be obtained using existing techniques. Specifically, the error of our algorithm depends on the maximal number of times that a {\em single user} holds a heavy-hitter as input, rather than the total number of times in which a heavy-hitter {\em exists}.
\end{abstract}

\end{titlepage}

\section{Introduction}

Differential privacy~\cite{DMNS06} is a mathematical definition for privacy, that aims to enable statistical analyses of databases while providing strong guarantees that individual-level information does not leak. More specifically, consider a database containing data pertaining to individuals, and suppose that we have some data analysis procedure that we would like to apply to this database.
We say that this procedure preserves {\em differential privacy} if no individual's data has a significant effect on the distribution of the outcome of the procedure. Intuitively, this guarantees that whatever is learned about an individual from the outcome of the computation could also be learned with her data arbitrarily modified (or without her data). Formally,

\begin{definition}[\cite{DMNS06}]\label{def:DP}
A randomized algorithm $\AAA$ is $(\eps,\delta)$-{\em differentially private} if for every two databases $S,S'$ that differ on one row, and every set of outcomes $F$, we have
$$\Pr[\AAA(S)\in F]\leq e^{\eps}\cdot \Pr[\AAA(S')\in F]+\delta.$$
\end{definition}

Over the last decade, we have witnessed an explosion of research on differential privacy, and by now it is largely accepted as a gold-standard for privacy preserving data analysis. 
One of the main properties that helped propel differential privacy to its current state is that differential privacy is closed under composition~\cite{DRV10}. This important property allows the design of differentially private algorithms to be modular: One can design complex differentially private algorithms by combining several (existing) differentially private algorithms and using them as subroutines. 
However, the downside with composition is that it degrades the privacy guarantees. Specifically, composing $k$ differentially private algorithms degrades the privacy guarantees by (roughly) a factor of $\sqrt{k}$. As a result, there are cases in which directly applying composition theorems is not the  way to obtain the best results, and different techniques are needed. In this work we revisit one of the most basic (and widely applicable) of these techniques, known as the {\em sparse vector technique}~\cite{DNRRV09}.

\paragraph{The sparse vector technique.}
Consider a large number of {\em counting queries} $f_1,\dots,f_m:X\rightarrow\{0,1\}$, which are given one by one to a data curator holding a database $S\in X^n$. In every round, we would like to receive, in a differentially private manner, an approximation to $f_i(S)\triangleq\sum_{x\in S}f_i(x)$. 
(A {\em counting query} asks for the number of records in $S$ that satisfy a given predicate, e.g.,\ the number of individuals with blue eyes and diabetes.) This task can be solved by answering each query in a differentially private manner (by adding an appropriately sampled noise), and using composition theorems to argue about the overall privacy guarantees. Due to the costs incurred by composition, this approach requires a database of size $n\gtrsim\sqrt{m}$ in order to answer $m$ such counting queries. In many cases, however, we might only be interested in obtaining answers for the queries $f_i$ whose answers are ``meaningful''. More specifically, we have some threshold $t$ in mind, and we only care to receive answers to queries s.t.\ $f_i(S)\geq t$. Dwork, Naor, Reingold, Rothblum, and Vadhan~\cite{DNRRV09} presented a simple (and elegant) tool for such a scenario -- algorithm \texttt{AboveThreshold}. This algorithm is capable of identifying (and answering) the {\em first} meaningful query from a stream of queries, while privately ignoring all the (non-meaningful) queries that came before it. The key observation of Dwork et al.\ is that privately ignoring the non-meaningful queries comes (essentially) for free. Informally, algorithm \texttt{AboveThreshold} can be described as follows (see Section~\ref{sec:prelims} for a precise description).

\begin{center}
\noindent\fboxother{
\parbox{.95\columnwidth}{
{\bf Algorithm \texttt{AboveThreshold} (informal)}\\[0.3em]
{\bf Input:} Database $S$, threshold $t$, and a stream of counting queries $f_i:X\rightarrow\{0,1\}$.
\begin{enumerate}[topsep=-1pt,rightmargin=5pt,itemsep=-1pt]%
\item Let $\hat{t}\leftarrow t+\Noise$.
\item In each round $i$, when receiving a query $f_i$, do the following:
\begin{enumerate}[topsep=-3pt,rightmargin=5pt]%
\item Let $\hat{f_i}\leftarrow f_i(S)+\Noise$.
\item If $\hat{f_i} \geq \hat{t}$, then output $\top$ and halt.
\item Otherwise, output $\bot$ and proceed to the next iteration.
\end{enumerate}
\end{enumerate}

}}
\end{center}

Once algorithm \texttt{AboveThreshold} identifies a ``meaningful'' query (and halts), estimating the value of this ``meaningful'' query can be done using standard differentially private noise addition techniques. In addition, by re-executing the algorithm after every such meaningful query (and using composition theorems to argue about the overall privacy guarantees), we can answer $c$ meaningful queries (out of a stream of $m$ queries) using a database of size $n=\tilde{O}\left(\sqrt{c}\cdot\log m\right)$. When the number of meaningful queries $c$ is much smaller than the total number of queries $m$ (hence the name {\em sparse vector}), this technique achieves a dramatic improvement compared to what we get when answering {\em all} queries. The sparse vector technique has proven itself to be one of the most important tools in the literature of differential privacy, with many applications (e.g.,~\cite{PMW_HR10,DFHPRR14,NSV16,NissimS18,KaplanSS20,KaplanLMNS20,HassidimKMMS20,BunSU16,BassilyTT18,CummingsKLZ19,LigettNRWW17,BartheFGGHS16,BartheGGHS16,SteinkeU16,CummingsKRW15,Ullman15,NandiB20,ShokriS15,hsu2013differential,sajed2019optimal,FeldmanS17,blum2015privacy}).

\subsection{Our contributions}

As we mentioned, algorithm \texttt{AboveThreshold} of Dwork et al.\ can be thought of as a technique for bypassing composition barriers (because instead of answering non-meaningful queries, paying in composition, we ``privately ignore'' them for free). However, in order to answer $c$ meaningful queries, we still use composition theorems when we re-execute the algorithm after every meaningful query. Do we lose something with this approach?

We present a variant of the algorithm \texttt{AboveThreshold} in which this transition between one meaningful query to the next is done ``more smoothly'', without re-executing the algorithm. While we cannot avoid the composition cost altogether, we account for it in a more fine-tuned way, which, in some cases, leads to a significant improvement. 
Recall that using algorithm \texttt{AboveThreshold} we can answer $c$ meaningful queries (out of a stream of $m$ queries) using a database of size $n=\tilde{O}\left(\sqrt{c}\cdot\log m\right)$.
Loosely speaking, with our algorithm, the factor $\sqrt{c}$ in the requirement on the database size is replaced with $\sqrt{k^*}$, where $k^*$ is the maximum number of times in which a {\em single} record in the database {\em contributes} to a ``meaningful'' answer. That is,
$$
k^*=\max_{x\in S}\left|\left\{  i : f_i(S)\geq t \text{ and } f_i(x)=1 \right\}\right|.
$$
Note that $k^*$ is always at most $c=\left|\left\{i : f_i(S)\geq t\right\}\right|$, and that it can be significantly smaller than $c$.

\subsubsection{Our algorithm}

At its most basic form, our algorithm can be informally described as follows. We remark that this presentation is oversimplified, and does not capture the full generality of our algorithm (in particular, the connection to $k^*$, mentioned above, will only become clear later).

\begin{center}
\noindent\fboxother{
\parbox{.95\columnwidth}{
{\bf Algorithm \texttt{ThresholdMonitor} (simplified)}\\[0.3em]
{\bf Input:} Database $S$, threshold $t$, and a stream of counting queries $f_i:X\rightarrow\{0,1\}$.
\begin{enumerate}[topsep=-1pt,rightmargin=5pt,itemsep=-1pt]%
\item In each round $i$, when receiving a query $f_i$, do the following:
\begin{enumerate}[topsep=-3pt,rightmargin=5pt]%
\item Let $\hat{f_i}\leftarrow f_i(S)+\Noise$.
\item If $\hat{f_i} < t $, then output $a_i=\bot$.
\item Otherwise, output $a_i=\top$ and delete from $S$ every $x$ such that $f_i(x)=1$. 
\end{enumerate}
\end{enumerate}

}}
\end{center}

The main difference from the standard \texttt{AboveThreshold} algorithm is that instead of halting after the first $\top$, the algorithm deletes all input elements $x\in S$ who ``contribute'' to the query that generated this $\top$, (that is, all $x\in S$ such that $f_i(x)=1$) and continues. %

\begin{remark}
In the actual construction, instead of deleting elements from the data after the first time they ``contribute'' to a $\top$ answer, we delete them only after their $k$th contribution (for a given parameter $k\in\N$). We show that this degrades the privacy guarantees of our algorithm by at most a factor of $\approx\sqrt{k}$. %
We also extend the construction to handle {\em linear} queries, i.e., queries that range over $[0,1]$ (rather than $\{0,1\}$). This means that the ``contribution'' of a data element to the value of a query is not binary as we present it here in the introduction.
\end{remark}

The privacy analysis of \texttt{ThresholdMonitor} builds on the privacy analysis of the standard \texttt{AboveThreshold} algorithm~\cite{DNRRV09}, in particular, on the analysis given by Hardt and Rothblum~\cite{PMW_HR10}. However, our analysis requires new ideas in order to show that by deleting some of the input elements we can avoid halting the algorithm after the first $\top$. We now give an informal (and simplified) presentation of the privacy analysis of \texttt{ThresholdMonitor}. 
Any informalities made hereafter will be removed in the sections that follow.

Fix two neighboring databases $S$ and $S'=S\cup\{x'\}$, and consider the execution of (the simplified) algorithm \texttt{ThresholdMonitor} on $S$ and on $S'$. Observe that once we reach an iteration $i^*$ in which $a_{i^*}=\top$ and $f_{i^*}(x')=1$ then $x'$ is deleted from the data (during the execution on $S'$). After this iteration, we have that $S$ and $S'$ are identical, and the two executions proceed exactly the same. Moreover, time steps $i\leq i^*$ in which $f_i(x')=0$ are also not very interesting, since in these time steps we have that $f_i(S)=f_i(S')$, and so the two executions behave the same.

Therefore, for the privacy analysis, we only need to argue about time steps $i\leq i^*$ such that $f_i(x')=1$.
At a high level, we partition these time steps into two sets:
$$I_{\rm far}=\{i\leq i^* : f_i(x')=1 \text{ and } f_i(S)\ll t\}, \qquad\text{and}\qquad I_{\rm almost}=\{i\leq i^* : f_i(x')=1 \}\setminus I_{\rm far}.$$ 
We show that, essentially, time steps in $I_{\rm far}$ come ``for free'', and we do not need to pay for them in the privacy analysis. Intuitively, the reason is that in these time steps we have that both $f_i(S)$ and $f_i(S')$ are very far from the threshold $t$, and in both executions, with overwhelming probability, we will get that $a_i=\bot$. 

For time steps in $I_{\rm almost}$ we do need to pay in the privacy analysis (by arguing about the noises that we add to the values of the queries). What we show, however, is that there cannot be too many  time steps in $I_{\rm almost}$. Intuitively, in every time step $i\in I_{\rm almost}$ there is a noticeable probability of returning $a_i=\top$, in which case $i=i^*$ is the last index in $I_{\rm almost}$. 

\medskip
\noindent
{\bf First challenge in the analysis.}
{\em
The partition of the time steps into $I_{\rm far}$ and $I_{\rm almost}$ depends on $i^*$, which depends on the noises that we add to the values of the queries at each time step. This means that conditioning on $i^*$ (or on $I_{\rm far}$ and $I_{\rm almost}$) changes the distribution of these noises. Can we still use them to argue about the privacy costs of the time steps in $I_{\rm almost}$?
}

\medskip
\noindent
{\bf Resolution.} 
We overcome this issue using the following trick. Instead of adding a single (Laplace) noise to the value of each query, we add {\em two} independent noises: $\hat{f}_i\leftarrow f_i(S)+w_i+v_i$, where $w_i$ and $v_i$ are independent random noises, sampled from the Laplace distribution. This allows us to define the partition of the time steps based on the $w_i$'s, while arguing about the privacy loss based on the $v_i$'s. So now time steps in $I_{\rm far}$ are time steps $i$ such that $f_i(S)+w_i\ll t$. Even when conditioning on that, $v_i$ is still distributed according to the Laplace distribution, and hence, can mask the difference between $f_i(S)$ and $f_i(S')$ (as is standard in the literature of differential privacy). This idea is general and could be applicable in other cases:
It  imposes a convenient-to-analyze structure on the probability spaces generated by the noises so that it is easier to argue in conditional subspaces.

\medskip
\noindent
{\bf Second challenge in the analysis.}
{\em
Do time steps in $I_{\rm far}$ really come for free?
}

\medskip
\noindent
{\bf Resolution.} 
As we described it, in these time steps, {\em with overwhelming probability} over sampling $v_i$ we have that $a_i=\bot$ (in both executions). This can be used to show that these time steps increase our privacy parameter $\delta$ by at most a factor of $m$, where $m$ bounds the total number of time steps. While this privacy loss is not too bad (since $\delta$ is typically negligible), we want to avoid it. In particular, the privacy guarantees of the standard \texttt{AboveThreshold} algorithm do not depend on the number of time steps $m$, and we would like to preserve this benefit in our analysis. To achieve this we modify the distribution of the noises $v_i$ (by capping them at a certain value) in order to make sure that for the vast majority of time steps $i\in I_{\rm far}$ we have that $a_i=\bot$ {\em with probability one} during both executions (over sampling $v_i$). This means that the vast majority of time steps in $I_{\rm far}$ really come for free (in terms of the privacy analysis). This allows us to avoid blowing up the privacy parameter $\delta$ by a factor of $m$.
In addition, we show that this modification to the distribution of the $v_i$'s does not break our privacy analysis for time steps $i\in I_{\rm almost}$.

\subsubsection{The shifting heavy-hitters problem} 

We demonstrate our technique by applying it to the {\em shifting heavy-hitters} problem. In this problem there are $n$ users and a data domain $X$. On every time step $i\in[m]$, every user $j\in[n]$ gets an input $x_{i,j}\in X$, and the goal is to report all the {\em current heavy-hitters}. That is, on time step $i$, the goal is to identify all data elements $x$ such that 
$$w_i(x)\triangleq\left|\left\{ j\in[n] : x_{i,j}=x  \right\}\right|\gg 0.$$

\begin{definition}
We say that an algorithm $\AAA$ solves the shifting heavy-hitters problem with error parameter $\tau$ and failure probability $\beta$ if it guarantees the following with probability at least $(1-\beta)$:
\begin{enumerate}
	\item The algorithm never identifies elements with weight zero as being heavy. That is, on every time step $i$, if the algorithm identifies an element $x$ then $w_i(x)>0$.
	\item On every time step $i$ the algorithm identifies every data element $x$ such that $w_i(x)\geq\tau$.
\end{enumerate}
\end{definition}

The shifting heavy-hitters problem generalizes the (standard) heavy-hitters problem, in which $m=1$. 
Algorithms for the (standard) heavy-hitters problem have found many applications in both the theory and the practice of differential privacy, including some of the most widespread industrial deployments of differential privacy to date. However, many of these deployments periodically recollect the data in order to re-identify the heavy-hitters,	e.g., on a daily basis. As a result, the problem that actually underlies these deployments is the {\em shifting heavy-hitters} problem. %

A naive approach for privately solving the shifting heavy-hitters problem would be to run a differential private algorithm for histograms on every time step $i\in[m]$. However, such a solution would have error at least $\sqrt{m}$ because of composition. 
This solution can be improved by using the (standard) sparse vector technique in order to identify the time steps in which a heavy-hitter exists, and only run a private algorithm for histograms during these identified time steps. However, in general, the number $k$ of time steps in which a heavy-hitter exists can be quite large, and the error with this approach scales with $\sqrt{k}$. We present a refined solution in which the error only scales with $\sqrt{k^*}$, where $k^*$ is the maximal number of times that a {\em single user} holds a heavy-hitter as input.

\begin{example}
To illustrate the difference between $k$ and $k^*$, consider an input stream to the shifting heavy-hitters problem with the following properties:
\begin{enumerate}
	\item There are three partitions of the set of users $[n]$:
	\begin{itemize}
		\item $A_1,A_2,\dots,A_{\tilde{O}(n)}\subseteq[n]$, where every $A_i$ is of size $\tilde{O}(1)$.
		\item $B_1,B_2,\dots,B_{\tilde{O}\left(n^{3/4}\right)}\subseteq[n]$, where every $B_i$ is of size $\tilde{O}\left(n^{1/4}\right)$.
		\item $C_1,C_2,\dots,C_{\tilde{O}\left(n^{2/3}\right)}\subseteq[n]$, where every $C_i$ is of size $\tilde{O}\left(n^{1/3}\right)$.
	\end{itemize}
	\item Every such subset gets a ``heavy-hitter'' exactly once throughout the stream, at different time steps.
	\item Otherwise, users get distinct inputs.
\end{enumerate} 
Observe that every user ``participates'' in three heavy-hitter occurrences, and hence, $k^*=3$. Indeed, in this example, our algorithm is capable of identifying {\em all} of the $\tilde{O}(n)$ heavy-hitters throughout the execution. In contrast, using the standard sparse vector technique only allows us to identify the heavy-hitters generated by the $C_i$'s. To see this, observe that in order to identify the $C_i$'s, the standard sparse vector technique needs to halt $\tilde{O}\left(n^{2/3}\right)$ times, which result in error at least $\tilde{O}\left(n^{1/3}\right)$ due to composition. This in turn means that this approach cannot identify the heavy-hitters generated by the $A_i$'s and the $B_i$'s.
\end{example}

\subsection{Related work}
Mostly related to our work are the works on the sparse vector technique and its variants~\cite{DNRRV09,PMW_HR10,ChenM15e,LyuSL17,BunSU16}. 
In particular, our analysis for algorithm \texttt{ThresholdMonitor} is based on techniques introduced by Hardt and Rothblum~\cite{PMW_HR10} in their analysis for the sparse vector technique. However, our analysis becomes more complicated in order to avoid halting the algorithm after the first $\top$ answer.

Another related work is the recent work of Feldman and Zrnic~\cite{FeldmanZ20}, who presented a refinement of standard composition theorems, in which they account separately for the privacy loss of different data elements, and delete elements from the data once they exhaust their privacy budget.

\section{Preliminaries}\label{sec:prelims}

\paragraph{Notation.}
Databases $S,S'$ are called {\em neighboring} if one is obtained from the other by adding or deleting one element, e.g., $S'=S\cup\{x'\}$.
For two random variables $Y,Z$ we write $X\approx_{(\eps,\delta)}Y$ to mean that for every event $F$ it holds that 
$\Pr[X\in F] \leq e^{\eps}\cdot\Pr[Y\in F]+\delta$, and $\Pr[Y\in F]\leq e^{\eps}\cdot\Pr[X\in F]+\delta$.

\paragraph{The Laplace mechanism.}
The most basic constructions of differentially private algorithms are via the Laplace mechanism as follows.

\begin{definition}[The Laplace Distribution]
A random variable has probability distribution $\Lap(b)$ if its probability density function is $f(x)=\frac{1}{2b}\exp\left(-\frac{|x|}{b}\right)$, where $x\in\R$.
\end{definition}

\begin{definition}[Sensitivity]
A function $f$ that maps databases to the reals has {\em sensitivity $\ell$} if for every neighboring $S,S'$, it holds that $|f(S)-f(S')|\leq \ell$.
\end{definition}

\begin{theorem}[The Laplace Mechanism \cite{DMNS06}]\label{thm:lap}
Let $f$ be a function that maps databases to the reals with sensitivity $\ell$. The mechanism $\AAA$ that on input $S$ adds noise with distribution $\Lap(\frac{\ell}{\eps})$ to the output of $f(S)$ preserves $(\eps,0)$-differential privacy.
\end{theorem}

\paragraph{Differential privacy for interactive mechanisms.}
In this work we consider {\em interactive} mechanisms that answer queries presented by a data analyst (or an adversary). 
Informally, an interactive mechanism $M$ is said to be differentially private if for every adversary $A$ that interacts with $M$ it holds that the algorithm that simulates $A$ and $M$ interacting, and then outputs the transcript of the interaction, is differentially private. We give here a definition presented by~\cite{BunSU16}.

\begin{definition}[Differential Privacy for Interactive Mechanisms~\cite{DMNS06,BunSU16}]\label{def:dpInteractive}
Let $M$ be a mechanism that takes an input database and answers adaptively chosen queries (by an adversary $A$ from some family $Q$ of possible queries).
Mechanism $M$ is $(\epsilon,\delta)$-differentially private if for every adversary $A$ we have that $\adaptivealg{A}{M}$ (defined below) is $(\epsilon,\delta)$-differentially private.
\end{definition}

\begin{algorithm*}[!ht]
\caption{\bf $\boldsymbol{\adaptivealg{A}{M}}$}\label{alg:adaptivealg}
{\bf Input:} Database $S$.
\begin{enumerate}[topsep=-1pt,rightmargin=5pt,itemsep=-1pt]

\item The database $S$ is given to the mechanism $M$.

\item For $i = 1,2,\dots$

\begin{enumerate}[topsep=-3pt,rightmargin=5pt]%
\item The adversary $A$ chooses a query $q_i \in Q$.
\item The mechanism $M$ is given $q_i$ and returns $a_i$.
\item $a_i$ is given to $A$.
\end{enumerate}

\item When $M$ halts or $A$ halts, output $(q_1,a_1,q_2,a_2,\cdots)$.

\end{enumerate}
\end{algorithm*}

\paragraph{Algorithm \texttt{AboveThreshold}.}
Consider an adaptively chosen stream of low-sensitivity functions $f_1,f_2,\ldots$ which are given one by one to a data curator (holding a database $S$). Algorithm \texttt{AboveThreshold} of Dwork et al.~\cite{DNRRV09} is a differentially private mechanism (the curator runs) for identifying the first index $i$ such that the value of $f_i(S)$ is ``large''. (I.e. the queries that it gets is whether $f_i(S)$ is larger than a threshold $t$) %

\begin{algorithm}[H]
\caption{\bf \texttt{AboveThreshold}}\label{alg:AboveThreshold}
{\bf Input:} Database $S\in X^*$, privacy parameter $\eps$, threshold $t$, and an adaptively chosen stream of sensitivity-1 queries $f_i:X^*\rightarrow\R$.
\begin{enumerate}[topsep=-1pt,rightmargin=5pt,itemsep=-1pt]%
\item Let $\hat{t}\leftarrow t+\Lap(\frac{2}{\eps})$.
\item In each round $i$, when receiving a query $f_i$, do the following:
\begin{enumerate}[topsep=-3pt,rightmargin=5pt]%
\item Let $\hat{f_i}\leftarrow f_i(S)+\Lap(\frac{4}{\eps})$.
\item If $\hat{f_i}\geq\hat{t}$, then output $\top$ and halt.
\item Otherwise, output $\bot$ and proceed to the next iteration.
\end{enumerate}
\end{enumerate}
\end{algorithm}

Notice that the number of possible rounds is unbounded. Nevertheless, this process preserves differential privacy:

\begin{theorem}[\cite{DNRRV09,PMW_HR10}]\label{thm:aThresh}
Algorithm \texttt{AboveThreshold} is $(\eps,0)$-differentially private.
\end{theorem}

\section{Algorithm \texttt{ThresholdMonitor} and its Analysis}\label{sec:ThresholdMonitor}

In this section we present out main result -- Algorithm \texttt{ThresholdMonitor}, specified in Algorithm~\ref{alg:ThresholdMonitor}. The utility properties of the algorithm are straightforward (follow form a simple union bound on the magnitude of the noises throughout the execution).

\begin{theorem}\label{thm:TM_utility}
Consider an execution of \texttt{ThresholdMonitor} on a database $S$ and on a stream of $m$ adaptively chosen queries: $f_1,f_2,\dots,f_m$. Let $S_i$ denote the database $S$ as it is before answering the $i$th query. With probability at least $1-\beta$,
\begin{enumerate}
	\item For every $i$ such that $a_i=\top$ it holds that $f_i(S_i)\geq t-\tilde{O}\left(\frac{1}{\eps}\log\left(\frac{1}{\delta}\right)\log\left(\frac{m}{\beta}\right)\right)$.
		\item For every $i$ such that $a_i=\bot$ it holds that $f_i(S_i)\leq t+\tilde{O}\left(\frac{1}{\eps}\log\left(\frac{1}{\delta}\right)\log\left(\frac{m}{\beta}\right)\right)$.
\end{enumerate} 
\end{theorem}

We now present the privacy guarantees of the algorithm.
Note that, as with \texttt{AboveThreshold}, the number of possible rounds unbounded. Nevertheless, the algorithm is differentially private:

\begin{theorem}\label{thm:privacymain}
Algorithm \texttt{ThresholdMonitor} is $(\xi,3\delta)$-differentially private, for $\xi=\frac{75(k+1)\eps}{\log\frac{1}{\delta}}+25\eps$. In particular, for $k=\log\frac{1}{\delta}$ we have that algorithm \texttt{ThresholdMonitor} is $(O(\eps),3\delta)$-differentially private.
\end{theorem}

\begin{proof}
Let $\BB:X^n\rightarrow\{\bot,\top\}^*$ denote the function that simulates \texttt{ThresholdMonitor} interacting with a given adversary $A$ (cf. algorithm~\ref{alg:adaptivealg}) and returns the transcript. Without loss of generality, we may assume that our adversary is deterministic.\footnote{Say that for every deterministic adversary we have that $\Pr[\BB(S)\in F]\leq e^\eps\cdot\Pr[\BB(S')\in F]+\delta$. Now for a random adversary, let $r$ denote a possible fixture of its random coins. It holds that $\Pr[\BB(S)\in F]=\sum_r \Pr[r] \Pr[\BB(S)\in F | r]\leq \sum_r \Pr[r] ( e^\eps \cdot\Pr[\BB(S')\in F | r]+\delta)=e^\eps\cdot\Pr[\BB(S')\in F]+\delta$.}
Fix two neighboring databases $S$ and $S'=S\cup\{x'\}$ for some $x'\in X$, and let $S_i$ (or $S'_i$) denote the database $S$ (or $S'$) as it is before answering the $i$th query during the execution on $S$ (or on $S'$). 
Also let $c_i(\cdot)$ denote the counter $c(\cdot)$ as it is after the $i$th iteration.
Denote $\Delta=\frac{1}{\eps}\log\left(\frac{1}{\delta}\right)\log\left(\frac{1}{\eps}\log\frac{1}{\delta}\right)$.
We need  the following definition.

\begin{algorithm*}[t!]
\caption{\bf \texttt{ThresholdMonitor}}\label{alg:ThresholdMonitor}

{\bf Input:} Database $S\in X^*$, privacy parameters $\eps,\delta$, threshold $t$, parameter $k$, 
and an adaptively chosen stream of counting queries $f_i:X\rightarrow[0,1]$.

\begin{enumerate}[leftmargin=15pt,rightmargin=10pt,itemsep=1pt,topsep=3pt]

\item Initiate a counter $c(x)=0$ for every $x\in X$.

\item Denote $\Delta=\frac{1}{\eps}\log\left(\frac{1}{\delta}\right)\log\left(\frac{1}{\eps}\log\frac{1}{\delta}\right)$.

\item In each round $i$, when receiving a query $f_i$, do the following:

	\begin{enumerate}[leftmargin=15pt,itemsep=1.5pt,topsep=1pt]
		\item Let $w_i\leftarrow\Lap\left(10\Delta \right)$ and let $v_i\leftarrow\Lap(\frac{1}{\eps}\log\frac{1}{\delta})$.
		
		\item Denote $\overline{v}_i=\min\left\{v_i,\Delta \right\}$.
		
		\item Let $\hat{f}_i = f_i(S) + w_i + \overline{v}_i$.

		\item If $\hat{f}_i< t$, then output $a_i=\bot$.
		
		\item Otherwise:
		
		\begin{itemize}[leftmargin=20pt,itemsep=1.5pt,topsep=1pt]
			\item Output $a_i=\top$.
			\item For every $x\in X$ set $c(x)\leftarrow c(x)+f_i(x)$.
			\item\label{step:delete} Delete from $S$ every element $x$ such that $c(x)\geq k$.
		\end{itemize}
		
		\item Proceed to the next iteration.

\end{enumerate}	
\end{enumerate}
\end{algorithm*}

\begin{definition}
We say that an {\em almost-top} occurs at time step $i$ if $a_i=\bot$ and $c_i(x')<k$ and 
$$f_i(S_i)+w_i\geq t - 2\Delta.$$
\end{definition}

Intuitively, almost-tops are time steps in which the noisy value of the query $\hat{f}_i$ {\em almost} crossed the threshold $t$, but did not, before the iteration in which $c(x')\geq k$. (This intuitive explanation is somewhat inaccurate because the definition of almost-tops only takes into account the noise $w_i$, without the noise $\overline{v}_i$.) We also identify the following subset of almost-tops, which we call {\em special-almost-tops}. In the introduction we mentioned that the vast majority of the time steps $i$ for which $f_i(S_i)+w_i\ll t$ come ``for free'' in terms of the privacy analysis. Intuitively, special-almost-tops represent the small portion of these time steps that do not come for free.

\begin{definition}
We say that a {\em special-almost-top} occurs at time step $i$ if $a_i=\bot$ and $c_i(x')<k$ and 
$$t-f_i(x')-\Delta \leq f_i(S_i)+w_i < t - \Delta.$$
\end{definition}

We define the following events.

\begin{center}
\noindent\fboxother{
\parbox{.9\columnwidth}{

\begin{align*}
\text{\bf Event ${\boldsymbol{E_1}}$:} &\qquad\qquad \sum_{i:\text{ almost-top}}f_i(x') \leq 15(k+1)+5{\cdot}\log\frac{1}{\delta}\\
\text{\bf Event ${\boldsymbol{E_2}}$:} &\qquad\qquad \left|\left\{i:\text{ special-almost-top}\right\}\right| \leq 30(k+1)+10{\cdot}\log\frac{1}{\delta}\\[1em]
\text{\bf Event ${\boldsymbol{E_3}}$:} &\qquad\qquad E_3=E_1\wedge E_2
\end{align*}
}}
\end{center}

\begin{lemma}\label{lem:event}
Consider the execution of $\BB(S)$. Event $E_3$ occurs with probability at least $1-3\delta$.
\end{lemma}

The formal proof of Lemma~\ref{lem:event} is given in Appendix~\ref{sec:ProofOfLemma}. Here we only sketch the argument.

\begin{proof}[Proof sketch]
Let $i^*$ denote the first time step such that $c_{i^*}(x')\geq k$.
Recall that each $w_i$ is sampled from $\Lap(10\Delta)$. Hence, in every time step $i$, the probability that $f_i(S_i)+w_i\geq t - 2\Delta$ is not much larger than the probability that $f_i(S_i)+w_i\geq t + 2\Delta$, in which case $a_i=\top$ with high probability (because $v_i$ is sampled from $\Lap(b)$ for $b\ll\Delta$). This means that the probability of an almost-top occurring in some time step $i$ is not much larger then the probability of a top occurring in this time step. Hence, it is very unlikely that $\sum_{i:\;\text{ almost-top}}f_i(x')$ would be much bigger than $\sum_{i\leq i^*:\; a_i=\top}f_i(x')$, which is bounded by $k+1$ by definition of $i^*$. This argument can be formalized to show that Event $E_1$ occurs with high probability.

For Event $E_2$, recall that for a time step $i$ to be a special-almost-top, the noise $w_i$ must satisfy $w_i \in \left[t-f_i(x')-\Delta, t-\Delta\right]$, which happens with probability at most $f_i(x')$, even conditioned on this time step being a (non-special) almost-top. (This is true by the properties of the Laplace distribution.) 
Therefore, the expected number of special-almost-tops is at most $\sum_{i:\text{ almost-top}}f_i(x')$, which is at most $15(k+1)+5{\cdot}\log\frac{1}{\delta}$ with high probability. Using standard concentration bounds, this shows that Event $E_2$ occurs with high probability.
\end{proof}

\medskip

We continue with the proof of Theorem~\ref{thm:privacymain}. 
Observe that, since we assume that our adversary is deterministic, every query $f_i$ is completely determined by the values of the previous answers $a_1,\dots,a_{i-1}$. We use the following notation.
 
\begin{notation}
Given a possible outcome vector $\vec{a}\in\{\bot,\top\}^*$, we write $f_{\vec{a},i}$ to denote the query chosen by the adversary after seeing the first $i-1$ answers in $\vec{a}$.
\end{notation}

An important thing to note is that  the outcome vector $\vec{a}$ and the noise vector $W=(w_1,w_2,\dots)$ completely determine whether $E_3$ occurs. 
The following lemma is a key in our analysis.

\begin{lemma}\label{lem:main}
The following holds for every outcome vector $\vec{a}\in\{\bot,\top\}^*$ and noise vector $W=(w_1,w_2,\dots)$ for which event $E_3$ occurs. 
Let $i^*\,{=}\,i^*(\vec{a})$ denote the first time step such that $c_{i^*}(x')\geq k$.
Let $I_{\rm top}=I_{\rm top}(\vec{a})$ denote the set of all time steps $i\leq i^*$ such that $a_i=\top$.
Let $W'$ denote a vector identical to $W$, except that every coordinate $i \in I_{\rm top}$ is decreased by $f_{\vec{a},i}(x')$. That is, 
$w'_{i}=w_{i}-f_{\vec{a},i}(x')$ for every $i\in I_{\rm top}$, and $w'_i=w_i$ for every other index $i$. Then,
$$
\Pr\left[\BB(S')=\vec{a} \left|  W' \right.\right]
\;\;\leq\;\;
\Pr\left[\BB(S)=\vec{a} \left|  W \right.\right]
\;\;\leq\;\;
\exp\left(\frac{75(k+1)\eps}{\log\frac{1}{\delta}}+25\eps\right)\cdot \Pr\left[\BB(S')=\vec{a} \left|  W \right.\right].
$$
(Note that since we fixed $W$ and $W'$, these probabilities are only over the noises $v_i$.)
\end{lemma}

We first show how Theorem~\ref{thm:privacymain} follows from Lemma~\ref{lem:main} (the proof of Lemma~\ref{lem:main} is given afterwards). For simplicity we will assume that the noise vector $W$ is discrete; the analysis could be extended to handle continues vectors $W$.\footnote{In fact, focusing on discrete vectors $W$ is without loss of generality, since we may round each coordinate of $W$, and each value $f_i(x)$, to the nearest multiplication of $1/N$, for some $N\geq n$. This changes the value of every query $f_i(S)$ by at most 1, which has basically no effect on the utility of the algorithm (since the noise we add to every $f_i(S)$ has variance much larger than 1).}
Let $F\subseteq\{\bot,\top\}^*$ be a set of possible outcomes. Denote $\xi= \frac{75(k+1)\eps}{\log\frac{1}{\delta}}+25\eps$. We have that

\begin{align}
&\Pr[\BB(S)\in F] \leq 3\delta +  \Pr[\BB(S)\in F \wedge E_3]%
=3\delta + \sum_{\vec{a}\in F} \Pr[\BB(S)=\vec{a} \wedge E_3]\nonumber\\
&\quad=3\delta + \sum_{\vec{a}\in F} \sum_{W} \Pr[W] \cdot \Pr[\BB(S)=\vec{a} \wedge E_3 \mid W]%
=3\delta + \sum_{\vec{a}\in F} \sum_{\substack{W:\\ E_3 \text{ occurs}\\ \text{for } \vec{a},W}} \Pr[W] \cdot \Pr[\BB(S)=\vec{a} \mid W]\nonumber\\
&\quad\leq3\delta + \sum_{\vec{a}\in F} \sum_{\substack{W:\\ E_3 \text{ occurs}\\ \text{for } \vec{a},W}} \Pr[W] \cdot e^{\xi}\cdot\Pr[\BB(S')=\vec{a} \mid W]\label{eq:explain0}\\
&\quad\leq3\delta + \sum_{\vec{a}\in F} \sum_{W} \Pr[W] \cdot  e^{\xi}\cdot\Pr[\BB(S')=\vec{a} \mid W]%
=3\delta+ e^{\xi}\cdot\Pr[\BB(S')\in F],\nonumber
\end{align}
where~(\ref{eq:explain0}) follows from Lemma~\ref{lem:main}.
The other direction is similar, except that we also need to argue about replacing $W$ with $W'$. Recall that each $W'$ is obtained from $W$ by decreasing $w_i$ by $f_{\vec{a},i}(x')$ for every coordinate $i\in I_{\rm top}$. Recall that by Step~\ref{step:delete} of the algorithm (and by definition of $I_{\rm top}$) we have that $\sum_{i\in I_{\rm top}} f_{\vec{a},i}(x') \leq k+1$. Also recall that every $w_i$ is sampled from $\Lap(10\Delta)$ for $\Delta=\frac{1}{\eps}\log\left(\frac{1}{\delta}\right)\log\left(\frac{1}{\eps}\log\frac{1}{\delta}\right)$. Therefore, for every such vector $W$ and its corresponding vector $W'$ we have
\begin{align}
\Pr[W]&=\left(\prod_{i\in I_{\rm top}}\Pr[w_i]\right)\cdot\left(\prod_{i\notin I_{\rm top}}\Pr[w_i]\right)%
\geq \left(\prod_{i\in I_{\rm top}}e^{-f_{\vec{a},i}(x')/(10\Delta)}\cdot\Pr[w'_i]\right)\cdot\left(\prod_{i\notin I_{\rm top}}\Pr[w_i]\right)\nonumber\\
&= \exp\left(-\frac{1}{10\Delta}\sum_{i\in I_{\rm top}} f_{\vec{a},i}(x')\right)\cdot \Pr[W']%
\geq \exp\left(-\frac{k+1}{10\Delta}\right)\cdot\Pr[W']%
\geq e^{-\xi}\cdot\Pr[W'].\label{eq:WW}
\end{align}
Hence,
\begin{align}
\Pr[\BB(S)\in F] &\geq \Pr[\BB(S)\in F \wedge E_3]%
= \sum_{\vec{a}\in F} \Pr[\BB(S)=\vec{a} \wedge E_3 ]\nonumber\\
&= \sum_{\vec{a}\in F} \sum_{W} \Pr[W]\cdot \Pr[\BB(S)=\vec{a} \wedge E_3 \mid W]%
= \sum_{\vec{a}\in F} \sum_{\substack{W:\\ E_3 \text{ occurs}\\ \text{for } \vec{a},W}} \Pr[W]\cdot \Pr[\BB(S)=\vec{a} \mid W]\nonumber\\
&\geq \sum_{\vec{a}\in F} \sum_{\substack{W:\\ E_3 \text{ occurs}\\ \text{for } \vec{a},W}} e^{-\xi}\cdot\Pr[W']\cdot \Pr[\BB(S')=\vec{a} \mid W']\label{eq:explain1}\\
&= \sum_{\vec{a}\in F} \sum_{\substack{W:\\ E_3 \text{ occurs}\\ \text{for } \vec{a},W'}} e^{-\xi}\cdot\Pr[W']\cdot \Pr[\BB(S')=\vec{a} \mid W']\label{eq:explain2}\\
&= e^{-\xi}\cdot\sum_{\vec{a}\in F} \sum_{W} \Pr[W']\cdot \Pr[\BB(S')=\vec{a} \wedge E_3 \mid W']%
=e^{-\xi}\cdot\sum_{\vec{a}\in F} \Pr[\BB(S')=\vec{a} \wedge E_3 ]\nonumber\\
&=e^{-\xi}\cdot\Pr[\BB(S')\in F \wedge E_3]\nonumber\\
&\geq e^{-\xi}\cdot \Pr[\BB(S')\in F \wedge E_3]+e^{-\xi}\cdot\Pr[\BB(S')\in F \wedge \overline{E_3}]-e^{-\xi}\cdot3\delta\nonumber\\
&= e^{-\xi}\cdot \Pr[\BB(S')\in F]-e^{-\xi}\cdot3\delta,\nonumber
\end{align}
where~(\ref{eq:explain1}) follows from Lemma~\ref{lem:main} and from Inequality~(\ref{eq:WW}), and where~(\ref{eq:explain2}) holds because $E_3$ occurs for $\vec{a},W$ if and only if it occurs for $\vec{a},W'$. Specifically, $E_3$ is determined from $\vec{a}$ (the set $I_{\rm top}$ is also determined from $\vec{a}$) and from $\{w_i : i\notin I_{\rm top}\}$, regardless of the value of the $w_{i}$'s in $I_{\rm top}$.
This completes the proof of Theorem~\ref{thm:privacymain} using Lemma~\ref{lem:main}. 
\end{proof}

It remains to prove Lemma~\ref{lem:main}.

\begin{proof}[Proof of Lemma~\ref{lem:main}]
Fix an outcome vector $\vec{a}$ and a vector of noises $W=(w_1,w_2,\dots)$ for which Event $E_3$ occurs. 
Let $i^*=i^*(\vec{a})$ denote the first time step such that $c_{i^*}(x')\geq k$. This is the time step in which $x'$ is deleted from the data. Denote $\Delta=\frac{1}{\eps}\log\left(\frac{1}{\delta}\right)\log\left(\frac{1}{\eps}\log\frac{1}{\delta}\right)$. 
Consider the following partition of the time steps (this partition depends on $\vec{a}$ and $W$):

\begin{center}
  \begin{tabular}{ l | l }
    \hline\hline\\[-0.8em]
    $I_{\rm top}$ & All indices $i\leq i^*$ such that $a_i=\top$ \\[0.4em] \hline \\[-0.8em]
		$I_{\rm after}$ & All indices $i>i^*$ \\[0.4em] \hline \\[-0.8em]
		$I_{\rm far}$ & All indices $i\leq i^*$ such that $a_i=\bot$ and $f_{\vec{a},i}(S)+w_i< t - f_{\vec{a},i}(x') - \Delta$ \\[0.4em] \hline \\[-0.8em]
		$I_{\rm special\text{-}almost}$ & All indices $i\leq i^*$ such that $a_i=\bot$ and $t-f_{\vec{a},i}(x')-\Delta \leq f_{\vec{a},i}(S_i)+w_i < t - \Delta$ \\[0.4em] \hline \\[-0.8em]
		$I_{\rm upper\text{-}almost}$ & All indices $i\leq i^*$ such that $a_i=\bot$ and $f_{\vec{a},i}(S)+w_i\geq t - \Delta$ \\[0.4em]
    \hline\hline
  \end{tabular}
\end{center}

Recall that we write $W'$ to denote a vector identical to $W$, except that coordinates $i\in I_{\rm top}$ are decreased by $f_{\vec{a},i}(x')$. 
For convenience, let $\BB=\BB(S)$ and $\BB'=\BB(S')$. 
We may now decompose

\begin{align}
\Pr\left[\BB(S)=\vec{a} \left|  W  \right.\right] &= \Pr\left[(\BB_1,\BB_2,\dots)=(a_1,a_2,\dots) \left|  W  \right.\right] \nonumber\\
&= \prod_{i=1,2,\dots} \Pr\left[\BB_i=a_i \left| (\BB_1,\dots,\BB_{i-1})=(a_1,\dots,a_{i-1}) \wedge  W  \right.\right]\nonumber\\
&= \prod_{i\in I_{\rm top}} \Pr\left[\BB_i=a_i \left| (\BB_1,\dots,\BB_{i-1})=(a_1,\dots,a_{i-1}) \wedge  W  \right.\right]\nonumber\\
&\qquad\cdot \prod_{i\in I_{\rm after}} \Pr\left[\BB_i=a_i \left| (\BB_1,\dots,\BB_{i-1})=(a_1,\dots,a_{i-1}) \wedge  W  \right.\right]\nonumber\\
&\qquad\cdot \prod_{i\in I_{\rm far}} \Pr\left[\BB_i=a_i \left| (\BB_1,\dots,\BB_{i-1})=(a_1,\dots,a_{i-1}) \wedge  W  \right.\right]\nonumber\\
&\qquad\cdot\prod_{i\in I_{\rm special\text{-}almost}} \Pr\left[\BB_i=a_i \left| (\BB_1,\dots,\BB_{i-1})=(a_1,\dots,a_{i-1}) \wedge  W  \right.\right]\nonumber\\
&\qquad\cdot\prod_{i\in I_{\rm upper\text{-}almost}} \Pr\left[\BB_i=a_i \left| (\BB_1,\dots,\BB_{i-1})=(a_1,\dots,a_{i-1}) \wedge  W  \right.\right]
\label{eq:decompose1}
\end{align}

We next analyze each of the expressions in Equality~(\ref{eq:decompose1}) separately, and relate them to the corresponding expressions with $S'$ instead of $S$. Recall that we need to bound these expressions both from above and from below.

\paragraph{Analysis for ${\boldsymbol{i\in I_{\rm top}}}$.} 

Recall that $S'=S\cup\{x'\}$, and that $f_{\vec{a},i}(S'_{i})= f_{\vec{a},i}(S_{i})+f_{\vec{a},i}(x')$ for every $i\in I_{\rm top}$. Therefore, for every $i\in I_{\rm top}$ we have

\begin{align*}
&\Pr\left[\BB_{i}=a_{i} \left| (\BB_1,\dots,\BB_{i-1})=(a_1,\dots,a_{i-1}) \wedge  W  \right.\right]\\
&\qquad\qquad=\Pr\left[\BB_{i}=\top \left| (\BB_1,\dots,\BB_{i-1})=(a_1,\dots,a_{i-1}) \wedge  W  \right.\right]\\
&\qquad\qquad=\Pr_{v_i}\left[ f_{\vec{a},i}(S_{i}) + w_{i} + \overline{v}_{i}  \geq t \right]\\
&\qquad\qquad=\Pr_{v_i}\left[ f_{\vec{a},i}(S'_{i}) - f_{\vec{a},i}(x') + w_{i} + \overline{v}_{i} \geq t \right]\\
&\qquad\qquad=\Pr\left[\BB'_{i}=a_{i} \left| (\BB'_1,\dots,\BB'_{i-1})=(a_1,\dots,a_{i-1}) \wedge  W'  \right.\right],
\end{align*}
Observe that one of the directions holds also without replacing $W$ with $W'$. Specifically,
\begin{align*}
&\Pr\left[\BB_{i}=a_{i} \left| (\BB_1,\dots,\BB_{i-1})=(a_1,\dots,a_{i-1}) \wedge  W  \right.\right]\\
&\qquad\qquad=\Pr\left[\BB_{i}=\top \left| (\BB_1,\dots,\BB_{i-1})=(a_1,\dots,a_{i-1}) \wedge  W  \right.\right]\\
&\qquad\qquad=\Pr_{v_i}\left[ f_{\vec{a},i}(S_{i}) + w_{i} + \overline{v}_{i}  \geq t \right]\\
&\qquad\qquad\leq\Pr_{v_i}\left[ f_{\vec{a},i}(S'_{i}) + w_{i} + \overline{v}_{i} \geq t \right]\\
&\qquad\qquad=\Pr\left[\BB'_{i}=a_{i} \left| (\BB'_1,\dots,\BB'_{i-1})=(a_1,\dots,a_{i-1}) \wedge  W  \right.\right].
\end{align*}

\paragraph{Analysis for ${\boldsymbol{i\in I_{\rm after}}}$.} Recall that during the execution on $S'$ we have that $x'$ is deleted from the input database during time $i=i^*$. Therefore, for every $i>i^*$ we have that $S_i=S'_i$. Hence, 

\begin{align*}
&\prod_{i\in I_{\rm after}} \Pr\left[\BB_i=a_i \left| (\BB_1,\dots,\BB_{i-1})=(a_1,\dots,a_{i-1}) \wedge  W  \right.\right]\\
&\qquad\qquad=\left(\prod_{\substack{i\in I_{\rm after}\\a_i=\bot}} \Pr_{v_i}\left[ f_{\vec{a},i}(S_i) + w_i + \overline{v}_i  < t \right]\right) \cdot \left(\prod_{\substack{i\in I_{\rm after}\\a_i=\top}} \Pr_{v_i}\left[ f_{\vec{a},i}(S_i) + w_i + \overline{v}_i  \geq t \right]\right)\\
&\qquad\qquad=\left(\prod_{\substack{i\in I_{\rm after}\\a_i=\bot}} \Pr_{v_i}\left[ f_{\vec{a},i}(S'_i) + w_i + \overline{v}_i  < t \right]\right) \cdot \left(\prod_{\substack{i\in I_{\rm after}\\a_i=\top}} \Pr_{v_i}\left[ f_{\vec{a},i}(S'_i) + w_i + \overline{v}_i  \geq t \right]\right)\\
&\qquad\qquad=\prod_{i\in I_{\rm after}} \Pr\left[\BB'_i=a_i \left| (\BB'_1,\dots,\BB'_{i-1})=(a_1,\dots,a_{i-1}) \wedge  W  \right.\right]\\
&\qquad\qquad=\prod_{i\in I_{\rm after}} \Pr\left[\BB'_i=a_i \left| (\BB'_1,\dots,\BB'_{i-1})=(a_1,\dots,a_{i-1}) \wedge  W' \right.\right],
\end{align*}
where the last equality is because once we fix $(a_1,\dots,a_{i-1})$ and $w_i$ then $\BB'_i$ is independent of $w_{j}$ for $j\neq i$ (it depends only on $v_i$).

\paragraph{Analysis for ${\boldsymbol{i\in I_{\rm far}}}$.} 

For every $i\in I_{\rm far}$ we have that 
$f_{\vec{a},i}(S_i)+w_i < t - f_{\vec{a},i}(x') - \Delta$. 
Therefore, we also have that
$$f_{\vec{a},i}(S'_i)+w_i = f_{\vec{a},i}(S_i)+f_{\vec{a},i}(x')+w_i  < t-\Delta.
$$
Recall that, by definition, $\overline{v}_i$ is always at most $\Delta$. Hence,

\begin{align*}
&\prod_{i\in I_{\rm far}} \Pr\left[\BB_i=a_i \left| (\BB_1,\dots,\BB_{i-1})=(a_1,\dots,a_{i-1}) \wedge  W  \right.\right]\\
&\qquad\qquad=\prod_{i\in I_{\rm far}} \Pr\left[\BB_i=\bot \left| (\BB_1,\dots,\BB_{i-1})=(a_1,\dots,a_{i-1}) \wedge  W  \right.\right]\\
&\qquad\qquad=\prod_{i\in I_{\rm far}} \Pr_{v_i}\left[ f_{\vec{a},i}(S_i) + w_i + \overline{v}_i  < t \right]\\
&\qquad\qquad=1\\
&\qquad\qquad=\prod_{i\in I_{\rm far}} \Pr_{v_i}\left[ f_{\vec{a},i}(S'_i) + w_i + \overline{v}_i  < t \right]\\
&\qquad\qquad=\prod_{i\in I_{\rm far}} \Pr\left[\BB'_i=a_i \left| (\BB'_1,\dots,\BB'_{i-1})=(a_1,\dots,a_{i-1}) \wedge  W  \right.\right]\\
&\qquad\qquad=\prod_{i\in I_{\rm far}} \Pr\left[\BB'_i=a_i \left| (\BB'_1,\dots,\BB'_{i-1})=(a_1,\dots,a_{i-1}) \wedge  W'  \right.\right].
\end{align*}

\paragraph{Analysis for ${\boldsymbol{i\in I_{\rm special\text{-}almost}}}$.} 
 
First observe that 

\begin{align*}
&\prod_{i\in I_{\rm special\text{-}almost}} \Pr\left[\BB_i=a_i \left| (\BB_1,\dots,\BB_{i-1})=(a_1,\dots,a_{i-1}) \wedge  W \right.\right]\\
&\qquad\qquad= \prod_{i\in I_{\rm special\text{-}almost}}\Pr\left[\BB_i=\bot \left| (\BB_1,\dots,\BB_{i-1})=(a_1,\dots,a_{i-1}) \wedge  W \right.\right]\\
&\qquad\qquad= \prod_{i\in I_{\rm special\text{-}almost}}\Pr_{v_i}\left[ f_{\vec{a},i}(S_i) + w_i + \overline{v}_i < t \right]\\
&\qquad\qquad\geq \prod_{i\in I_{\rm special\text{-}almost}}\Pr_{v_i}\left[ f_{\vec{a},i}(S'_i) + w_i + \overline{v}_i < t \right]\\
&\qquad\qquad=\prod_{i\in I_{\rm special\text{-}almost}}\Pr\left[\BB'_i=a_i \left| (\BB'_1,\dots,\BB'_{i-1})=(a_1,\dots,a_{i-1}) \wedge  W \right.\right]\\
&\qquad\qquad=\prod_{i\in I_{\rm special\text{-}almost}}\Pr\left[\BB'_i=a_i \left| (\BB'_1,\dots,\BB'_{i-1})=(a_1,\dots,a_{i-1}) \wedge  W' \right.\right].
\end{align*}

For the other direction, fix $i\in I_{\rm special\text{-}almost}$. Recall that $v_i$ is sampled from $\Lap(\frac{1}{\eps}\log\frac{1}{\delta})$, and that we denote $\Delta=\frac{1}{\eps}\log\left(\frac{1}{\delta}\right)\log\left(\frac{1}{\eps}\log\frac{1}{\delta}\right)$. Using the fact that $e^{\gamma}\cdot\left(1-\frac{\gamma}{2}\right)\geq1$ for every $0\leq\gamma\leq1$ we have that

\begin{align}
&\Pr\left[\BB_i=a_i \left| (\BB_1,\dots,\BB_{i-1})=(a_1,\dots,a_{i-1}) \wedge  W \right.\right]\nonumber\\
&\qquad\qquad\leq1\nonumber\\
&\qquad\qquad\leq e^{\eps/\log\frac{1}{\delta}}\cdot\left(1-\frac{\eps}{2\log\frac{1}{\delta}}\right)\nonumber\\
&\qquad\qquad\leq e^{\eps/\log\frac{1}{\delta}}\cdot \Pr_{v_i}\left[ v_i < \Delta \right]\nonumber\\
&\qquad\qquad\leq e^{\eps/\log\frac{1}{\delta}}\cdot \Pr_{v_i}\left[ f_{\vec{a},i}(S_i) + w_i + v_i < t \right]\label{eq:explain3}\\
&\qquad\qquad\leq e^{\eps/\log\frac{1}{\delta}}\cdot e^{\eps/\log\frac{1}{\delta}}\cdot \Pr_{v_i}\left[ f_{\vec{a},i}(S'_i) + w_i + v_i < t \right]\nonumber\\
&\qquad\qquad\leq e^{2\eps/\log\frac{1}{\delta}}\cdot \Pr_{v_i}\left[ f_{\vec{a},i}(S'_i) + w_i + \overline{v}_i < t \right]\nonumber\\
&\qquad\qquad=e^{2\eps/\log\frac{1}{\delta}}\cdot \Pr\left[\BB'_i=a_i \left| (\BB'_1,\dots,\BB'_{i-1})=(a_1,\dots,a_{i-1}) \wedge  W \right.\right],\nonumber
\end{align}
where (\ref{eq:explain3}) follows since for $i\in I_{\rm special\text{-}almost}$ we have that 
$f_{\vec{a},i}(S_i)+w_i +\Delta< t$, and hence, if $v_i < \Delta$ then $f_{\vec{a},i}(S_i) + w_i + v_i<t$.
Now, by Event $E_2$ we have that $\left| I_{\rm special\text{-}almost} \right|\leq 30(k+1) + 10\log\frac{1}{\delta}$, and therefore,

\begin{align*}
&\prod_{i\in I_{\rm special\text{-}almost}} \Pr\left[\BB_i=a_i \left| (\BB_1,\dots,\BB_{i-1})=(a_1,\dots,a_{i-1}) \wedge  W \right.\right]\\
&\qquad\qquad\leq \prod_{i\in I_{\rm special\text{-}almost}} e^{2\eps/\log\frac{1}{\delta}}\cdot \Pr\left[\BB'_i=a_i \left| (\BB'_1,\dots,\BB'_{i-1})=(a_1,\dots,a_{i-1}) \wedge  W \right.\right]\\
&\qquad\qquad\leq \exp\left(\frac{60(k+1)\eps}{\log\frac{1}{\delta}}+20\eps\right)\cdot\prod_{i\in I_{\rm special\text{-}almost}} \Pr\left[\BB'_i=a_i \left| (\BB'_1,\dots,\BB'_{i-1})=(a_1,\dots,a_{i-1}) \wedge  W \right.\right].
\end{align*}

\paragraph{Analysis for ${\boldsymbol{i\in I_{\rm upper\text{-}almost}}}$.}

As in the analysis for $i\in I_{\rm special\text{-}almost}$, we have that

\begin{align*}
&\prod_{i\in I_{\rm upper\text{-}almost}} \Pr\left[\BB_i=a_i \left| (\BB_1,\dots,\BB_{i-1})=(a_1,\dots,a_{i-1}) \wedge  W \right.\right]\\
&\qquad\qquad\geq\prod_{i\in I_{\rm upper\text{-}almost}}\Pr\left[\BB'_i=a_i \left| (\BB'_1,\dots,\BB'_{i-1})=(a_1,\dots,a_{i-1}) \wedge  W \right.\right]\\
&\qquad\qquad=\prod_{i\in I_{\rm upper\text{-}almost}}\Pr\left[\BB'_i=a_i \left| (\BB'_1,\dots,\BB'_{i-1})=(a_1,\dots,a_{i-1}) \wedge  W' \right.\right].
\end{align*}

For the other direction,

\begin{align}
&\prod_{i\in I_{\rm upper\text{-}almost}} \Pr\left[\BB_i=a_i \left| (\BB_1,\dots,\BB_{i-1})=(a_1,\dots,a_{i-1}) \wedge  W \right.\right]\nonumber\\
&\qquad= \prod_{i\in I_{\rm upper\text{-}almost}} \Pr\left[\BB_i=\bot \left| (\BB_1,\dots,\BB_{i-1})=(a_1,\dots,a_{i-1}) \wedge  W \right.\right]\nonumber\\
&\qquad= \prod_{i\in I_{\rm upper\text{-}almost}} \Pr_{v_i}\left[ f_{\vec{a},i}(S_i) + w_i + \overline{v}_i < t \right]\nonumber\\
&\qquad=\prod_{i\in I_{\rm upper\text{-}almost}} \Pr\left[ f_{\vec{a},i}(S_i) + w_i + v_i < t \right]\label{eq:explain4}\\
&\qquad\leq \prod_{i\in I_{\rm upper\text{-}almost}} e^{f_{\vec{a},i}(x')\cdot\eps/\log\frac{1}{\delta}}\cdot \Pr\left[ f_{\vec{a},i}(S_i) + f_{\vec{a},i}(x') + w_i + v_i < t \right]\nonumber\\
&\qquad= \prod_{i\in I_{\rm upper\text{-}almost}} e^{f_{\vec{a},i}(x')\cdot\eps/\log\frac{1}{\delta}}\cdot \Pr\left[ f_{\vec{a},i}(S'_i) + w_i + v_i < t \right]\nonumber\\
&\qquad\leq \prod_{i\in I_{\rm upper\text{-}almost}} e^{f_{\vec{a},i}(x')\cdot\eps/\log\frac{1}{\delta}}\cdot \Pr\left[ f_{\vec{a},i}(S'_i) + w_i + \overline{v}_i < t \right]\nonumber\\
&\qquad=\prod_{i\in I_{\rm upper\text{-}almost}} e^{f_{\vec{a},i}(x')\cdot\eps/\log\frac{1}{\delta}}\cdot \Pr\left[\BB'_i=a_i \left| (\BB'_1,\dots,\BB'_{i-1})=(a_1,\dots,a_{i-1}) \wedge  W \right.\right]\nonumber\\
&\qquad=\exp\left(\sum_{i\in I_{\rm upper\text{-}almost}}\frac{f_{\vec{a},i}(x')\cdot\eps}{\log(1/\delta)}\right)\cdot\prod_{i\in I_{\rm upper\text{-}almost}} \Pr\left[\BB'_i=a_i \left| (\BB'_1,\dots,\BB'_{i-1})=(a_1,\dots,a_{i-1}) \wedge  W \right.\right]\nonumber\\
&\qquad\leq\exp\left(\frac{15\eps(k+1)}{\log(1/\delta)}+5\eps\right)\cdot\prod_{i\in I_{\rm upper\text{-}almost}} \Pr\left[\BB'_i=a_i \left| (\BB'_1,\dots,\BB'_{i-1})=(a_1,\dots,a_{i-1}) \wedge  W \right.\right],\label{eq:explain5}
\end{align}
where (\ref{eq:explain5}) follows from Event $E_1$, and where (\ref{eq:explain4}) holds since $\overline{v}_i=\min\left\{v_i,\; \Delta \right\}$ and since for $i\in I_{\rm upper\text{-}almost}$ we have $f_{\vec{a},i}(S)+w_i+\Delta\geq t$. Hence, if $f_{\vec{a},i}(S_i) + w_i + \overline{v}_i < t$ then $f_{\vec{a},i}(S_i) + w_i + v_i < t$.

\medskip

This concludes the analysis of the different expressions in Equality~(\ref{eq:decompose1}).
Returning to Equality~(\ref{eq:decompose1}), we now get that

$$
\Pr\left[\BB(S')=\vec{a} \left|  W' \right.\right]
\;\;\leq\;\;
\Pr\left[\BB(S)=\vec{a} \left|  W \right.\right]
\;\;\leq\;\;
\exp\left(\frac{75(k+1)\eps}{\log\frac{1}{\delta}}+25\eps\right) \cdot \Pr\left[\BB(S')=\vec{a} \left|  W \right.\right].
$$
This completes the proof of Lemma~\ref{lem:main}.
\end{proof}

\subsection[Composition for k tops]{Composition for ${\boldsymbol k}$ tops}\label{sec:composition}

In Theorem~\ref{thm:privacymain} we showed that when executing algorithm \texttt{ThresholdMonitor} with a parameter $k$, then the privacy guarantees of the algorithm degrade at most linearly with $k/\log\frac{1}{\delta}$. We give a refined analysis, showing that the privacy guarantees degrade only (roughly) proportionally to $\sqrt{k}$. This refined analysis follows 
by applying known composition theorems for differential privacy~\cite{DRV10} to the construction of algorithm \texttt{ThresholdMonitor}.
Although standard, this refined analysis requires some technical work, as we need to ``break'' algorithm \texttt{ThresholdMonitor} into smaller components in order to ``compose'' it more efficiently using composition theorems. The details are given in Appendix~\ref{sec:ProofOfComposition}. Here we only state the result.

\begin{theorem}\label{thm:kprivacy}
Algorithm \texttt{ThresholdMonitor} with parameter $k$ is $(\eps_0,3k\delta)$-differentially private for $\eps_0=O\left(\sqrt{k}\eps+k\eps^2\right)$. 
\end{theorem}

\section{The Shifting Heavy-Hitters Problem}\label{sec:shh}

In the {\em shifting heavy-hitters} problem there are $n$ users and a data domain $X$. On every time step $i\in[m]$, every user $j$ gets an input $x_{i,j}\in X$, and the goal is to report all the {\em current heavy-hitters}. That is, on time step $i$, the goal is to identify all data elements $x$ such that 
$$w_i(x)\triangleq\left|\left\{ j\in[n] : x_{i,j}=x  \right\}\right|\gg 0.$$

\begin{definition}
We say that an algorithm $\AAA$ solves the shifting heavy-hitters problem with error parameter $\tau$ and failure probability $\beta$ if it guarantees the following with probability at least $(1-\beta)$:
\begin{enumerate}
	\item The algorithm never identifies elements with weight zero as being heavy. That is, on every time step $i$, if the algorithm identifies an element $x$ then $w_i(x)>0$.
	\item On every time step $i$ the algorithm identifies every data element $x$ such that $w_i(x)\geq\tau$.
\end{enumerate}
\end{definition}

Of course, we want to solve the shifting heavy-hitters problem while guaranteeing differential privacy. However, as users' data evolve over time, algorithms for the shifting heavy-hitters problem do not exactly fit the standard (non-interactive) definition of differential privacy (Definition~\ref{def:DP}), nor the interactive variant (Definition~\ref{def:dpInteractive}) in which the queries are chosen adaptively, but users' data do not change. 
We will use the following definition of~\cite{BNS20_compo} for what it means for a mechanism to be differentially private when the data evolves over time (generalizing the standard definition of differential privacy~\cite{DMNS06}).

\begin{notation}
Two databases $S,S'$ (of the same size) are called $j$-neighboring if they agree on every coordinate $i\neq j$.
\end{notation}

\begin{definition}[Differential Privacy for Interactive Mechanisms with Evolving Data~\cite{DMNS06,BNS20_compo}]\label{def:dpEv}
Let $M$ be a mechanism with the following properties. On each round $i\in[m]$, the mechanism gets a database $S_i$ (of size $n$) and a query $q_i$ (from some family $Q$ of possible queries), and responds with an answer $a_i$. Mechanism $M$ is $(\epsilon,\delta)$-differentially private in the evolving setting if for every adversary $A$ and every $j\in[n]$ we have that 
$\texttt{Game}_{A,M,j}$, defined below, is $(\epsilon,\delta)$-differentially private.
\end{definition}

\begin{algorithm*}[!ht]
\caption{\bf $\boldsymbol{\texttt{Game}_{A,M,j}}$}
{\bf Input:} A bit $b\in\{0,1\}$. (The bit $b$ is unknown to $A$ and $M$.)
\begin{enumerate}[topsep=-1pt,rightmargin=5pt,itemsep=-1pt]

\item For $i = 1,2,\dots,m$

\begin{enumerate}[topsep=-3pt,rightmargin=5pt]%
\item The adversary $A$ chooses two $j$-neighboring databases $S_i^0,S_i^1$ and a query $q_i \in Q$.
\item The mechanism $M$ is given $S_i^b$ and $q_i$ and returns $a_i$.
\item $a_i$ is given to $A$.
\end{enumerate}

\item Output $A$'s view of the game, that is $S_1^0,S_1^1,q_1,a_1,\dots,S_m^0,S_m^1,q_m,a_m$, and the internal randomness of $A$.\vspace{5px}
\end{enumerate}

\end{algorithm*}

Towards designing a private algorithm for the shifting heavy-hitters problem, we generalize algorithm \texttt{ThresholdMonitor} for the evolving data setting. Consider algorithm \texttt{ThresholdMonitorEvolving}, or \texttt{TME} in short, given in Figure~\ref{alg:ThresholdMonitorEvolving}. Observe that in the special case where $S_i=S$ for all $i$ (i.e., when the database does not change over time) then this algorithm is identical to \texttt{ThresholdMonitor}.

\begin{algorithm*}[t!]
\caption{\bf \texttt{ThresholdMonitorEvolving}}\label{alg:ThresholdMonitorEvolving}

{\bf Input:} Privacy parameters $\eps,\delta$, threshold $t$, and parameter $k$.

{\bf Setting:} On every time step $i\in[m]$ we get a database $S_i\in X^n$, and a query $f_i:X\rightarrow[0,1]$.

\begin{enumerate}[leftmargin=15pt,rightmargin=10pt,itemsep=1pt,topsep=3pt]

\item Initiate a counter $c(j)=0$ for every $j\in [n]$.

\item Denote $\Delta_1=O\left(\frac{\sqrt{k}}{\eps}\log\left(\frac{k}{\delta}\right)\log\left(\frac{k}{\eps}\log\frac{k}{\delta}\right)\right)$ and $\Delta_2=O\left(\frac{\sqrt{k}}{\eps}\log\left(\frac{k}{\delta}\right)\right)$.

\item In each round $i$, when receiving a database $S_i=(x_{i,1},\dots,x_{i,n})$ and a query $f_i$, do the following:

	\begin{enumerate}[leftmargin=15pt,itemsep=1.5pt,topsep=1pt]
		\item Let $w_i\leftarrow\Lap\left(10\Delta_1 \right)$ and let $v_i\leftarrow\Lap\left(\Delta_2\right)$.
		
		\item Denote $\overline{v}_i=\min\left\{v_i,\Delta_1 \right\}$.
		
		\item Let $\hat{f}_i = \left(\sum_{j\in[n]: c(j)<k} f_i(x_{i,j})\right) + w_i + \overline{v}_i$.

		\item If $\hat{f}_i< t$, then output $a_i=\bot$.
		
		\item Otherwise:
		
		\begin{itemize}[leftmargin=20pt,itemsep=1.5pt,topsep=1pt]
			\item Output $a_i=\top$.
			\item For every $j\in [n]$ set $c(j)\leftarrow c(j)+f_i(x_{i,j})$.
		\end{itemize}
		
		\item Proceed to the next iteration.

\end{enumerate}	
\end{enumerate}
\end{algorithm*}

\begin{theorem}\label{thm:privacymainEvolving}
Algorithm \texttt{ThresholdMonitorEvolving} is $(\eps,\delta)$-differentially private in the evolving setting.
\end{theorem}

Recall that an adversary for algorithm \texttt{ThresholdMonitorEvolving} chooses (on every time step) both the next query and the next inputs to the users, while an adversary for algorithm \texttt{ThresholdMonitor} only chooses the queries. 
To prove Theorem~\ref{thm:privacymainEvolving}, we will show that the adversary's choices for the users' inputs can be encoded in the choice for the next query, without actually changing users' inputs. In other words, will map an execution of \texttt{ThresholdMonitorEvolving} (in which the data evolves over time) into an execution of \texttt{ThresholdMonitor} in which the data remains fixed. 

\begin{proof}[Proof of Theorem~\ref{thm:privacymainEvolving}]
Fix an adversary $A$ and an index $j\in[n]$. We need to show that algorithm $\texttt{Game}_{A,\texttt{TME},j}$ is $(\eps,\delta)$-differentially private. To that end, let us consider a fictional algorithm $\texttt{FicGame}_{A,\texttt{TME},j}$ in which the data given to \texttt{TME} does not evolve over time. (Algorithm $\texttt{FicGame}_{A,\texttt{TME},j}$ is only part of the analysis.)

\begin{algorithm*}[!ht]
\caption{\bf $\boldsymbol{\texttt{FicGame}_{A,\texttt{TME},j}}$}
{\bf Input:} A bit $b\in\{0,1\}$.
\begin{enumerate}[topsep=-1pt,rightmargin=5pt,itemsep=-1pt]

\item Let $D=(d_1,\dots,d_n)\in([n]\times\{0,1\})^n$ be such that $d_j=(j,b)$ and $d_r=(r,0)$ for every $r\neq j$.

\item Instantiate algorithm \texttt{ThresholdMonitor} on the database $D$ with parameters $t,k$ and with privacy parameters $\tilde{\eps}\leftarrow O\left( \eps/\sqrt{k} \right)$ and $\tilde{\delta}\leftarrow O\left(\delta/k\right)$.

\item For $i = 1,2,\dots,m$

\begin{enumerate}[topsep=-3pt,rightmargin=5pt]%
\item The adversary $A$ chooses two $j$-neighboring databases $S_i^0=(x_{i,1}^0,\dots,x_{i,n}^0)$ and $S_i^1=(x_{i,1}^1,\dots,x_{i,n}^1)$ and a query $q_i:X\rightarrow[0,1]$.
\item Define $f_i:([n]\times\{0,1\})\rightarrow[0,1]$ as follows: $f_i(r,z)=q_i(x_{i,r}^z)$.
\item Give $f_i$ to algorithm \texttt{ThresholdMonitor} and obtain an answer $a_i$.
\item $a_i$ is given to $A$.
\end{enumerate}

\item Output $A$'s view of the game, that is $S_1^0,S_1^1,q_1,a_1,\dots,S_m^0,S_m^1,q_m,a_m$, and the internal randomness of $A$.\vspace{5px}
\end{enumerate}

\end{algorithm*}

Observe that by Theorem~\ref{thm:kprivacy} we have that algorithm $\texttt{FicGame}_{A,\texttt{TME},j}$ is $(\eps,\delta)$-differentially private (w.r.t.\ the bit $b$). In addition, for every iteration $i$ and for every $r\in[n]$ we have that $f_i(d_r)=q_i(x_{i,r}^b)$. Hence, the answers $a_i$ are distributed exactly the same during the execution of $\texttt{FicGame}_{A,\texttt{TME},j}$ and $\texttt{Game}_{A,\texttt{TME},j}$. As a result, the view of the adversary $A$ is also distributed the same. Therefore, algorithm $\texttt{Game}_{A,\texttt{TME},j}$ is $(\eps,\delta)$-differentially private.
\end{proof}

We now return to the shifting heavy-hitters problem.  
A technical issue here is that there is a certain circularity between $\tau$ (the error parameter of the algorithm) and  $k$ (the number of times in which a heavy-hitter exists). The reason is that $\tau$ defines a threshold such the algorithm must identify every element with weight at least $\tau$ as being a ``heavy-hitter''. Therefore, when $\tau$ gets smaller, then there might be more time steps in which a heavy-hitter exists, which increases $k$. This, in turn, increases $\tau$, because the error of the algorithm grows with $k$.

\begin{notation}
For $k\in\N$ we denote
$$
\tau(k)=\tilde{O}\left( \frac{\sqrt{k}}{\eps} \cdot \log\left(\frac{1}{\delta}\right) \cdot \log\left(\frac{m\cdot|X|}{\beta}\right) \right). 
$$
\end{notation}

\begin{notation}\label{not:k*}
Let $k^*$ be such that the maximal number of times that a single user holds an element with weight $\tau(k^*)$ is at most $k^*$. That is,
$$
k^*\geq \max_{j\in[n]}  \left|\left\{  i\in[m] : w_i(x_{i,j})\geq \tau(k^*)  \right\}\right|.
$$
\end{notation}

We present an algorithm for the shifting heavy-hitters problem in which the error scales roughly as $\approx\sqrt{k^*}$.

\begin{theorem}\label{thm:shh}
There exists an $(\eps,\delta)$-differentially private algorithm for solving the shifting heavy-hitters problem with failure probability $\beta$ and error $O\left(\tau(k^*)\right)$.
\end{theorem}

We remark that the assumption on $k^*$ is only used in the utility analysis, and that the privacy of our algorithm is guaranteed even if this assumption does not hold.

\begin{proof}
We use \texttt{ThresholdMonitorEvolving} with parameters $t=O\left(\tau(k^*)\right)$ and $k^*$ in order to identify the heavy hitters on every time step. Specifically, on every time step $i\in[m]$ we query \texttt{ThresholdMonitorEvolving} with the functions $f_x(y)=\1_{\{y=x\}}$ for every $x\in X$, and report $x$ as a heavy-hitter if the corresponding answer is $\top$. The privacy properties of this algorithm follow directly from the privacy properties of \texttt{ThresholdMonitorEvolving} (see Theorem~\ref{thm:privacymainEvolving}). The utility guarantees follow from the fact that throughout the execution we sample at most $O(m\cdot|X|)$ Laplace random variables. By the properties of the Laplace distribution, with probability at least $1-\beta$, all of these random variables are at most $\tau(k^*)$ in absolute value (see Theorem~\ref{thm:TM_utility}). 
\end{proof}

\section*{Acknowledgments}
The authors are grateful to Guy Rothblum for many helpful discussions.

\bibliographystyle{abbrv}

\appendix

\section{Proof of Lemma~\ref{lem:event}}\label{sec:ProofOfLemma}

In this section we prove Lemma~\ref{lem:event}, which states that Event $E_3$ occurs with high probability. 
As we mentioned, the intuitive explanation is that in every round $i$, the probability of an almost-top is roughly the same as the probability of an actual top, and hence, the almost-tops and the actual tops should be ``balanced''. The analysis we give here is an adaptation of the analysis of Gupta et al.~\cite{GuptaLMRT10} for a different scenario.

Towards proving Lemma~\ref{lem:event}, let us consider the following $m$-round game against an adversary. 

\begin{center}
\noindent\fboxother{
\parbox{.9\columnwidth}{
{\bf A $\boldsymbol{m}$-round game}\\[0.7em]
In each round $i$:
\begin{enumerate}
	\item The adversary chooses $0\leq p_i\leq\frac{1}{2}$ and $\frac{p_i}{4}\leq q_i\leq1-p_i$, and $\gamma_i\in[0,1]$, possibly based on the first $(i-1)$ rounds.
	\item A random variable $X_i\in\{0,1,2\}$ is sampled (and the outcome is given to the adversary), where $\Pr[X_i=1]=p_i$ and $\Pr[X_i=2]=q_i$ and $\Pr[X_i=0]=1-p_i-q_i$.
\end{enumerate}
}}
\end{center}

We refer to the values $\gamma_i$ as {\em weights}.
For $k\in\R$ and $i\in[m]$, let $Z_i^{(k)}$ be the indicator for the event that 
$$
\sum_{j =1}^{i}\1\{X_j=2\}\cdot\gamma_j \leq k.
$$
That is, $Z_i^{(k)}=1$ if and only if the sum of weights $\gamma_j$ for time steps $1\leq j \leq i$ such that $X_j=2$ is at most $k$.
Note that for $k<0$ we have that $Z_i^{(k)}\equiv0$.
For $k\in\R$ let
$$W^{(k)}=\sum_{i=1}^m \1\{X_i=1\}\cdot\gamma_i\cdot Z_i^{(k)}.$$ 
This can be interpreted as follows. Let $i^*$ denote the first time step in which $Z_{i^*}^{(k)}=0$. Then $W^{(k)}$ is the sum of weights $\gamma_i$ for time steps $1\leq i<i^*$ in which $X_i=1$.

\begin{remark}
Intuitively, we can think of $k$ as the ``budget'' of the adversary -- every time that $X_i=2$ then the budget is decreased by $\gamma_i$.
The adversary's goal is to maximize the sum of weights $\gamma_i$ for time steps in which $X_i=1$ {\em before} the budget runs out. That is, every time that $X_i=1$ the adversary gets $\gamma_i$ as a ``reward'' (before the budget runs out), and every time that $X_i=2$ the budget is decreased by $\gamma_i$. The random variable $W^{(k)}$ represents the sum of rewards obtained by the adversary before the budget runs out.
\end{remark}

Towards analyzing $W^{(k)}$, we define the following partial sums. For $j\in[m]$ and $k\in\R$, let
$$W_j^{(k)}=\sum_{i=j}^m \1\{X_i=1\}\cdot\gamma_i\cdot Z_i^{(k)}.$$ 
Observe that $W^{(k)}=W_1^{(k)}$.

\begin{claim}\label{claim:game}
For every adversary's strategy, every $k\geq0$, every $\lambda\in\R$, and every $j\in[m]$ we have that $\Pr[W^{(k)}_j>\lambda]\leq\exp\left(-\frac{\lambda}{5}+3(k+1)\right)$. In particular, $\Pr[W^{(k)}>\lambda]\leq\exp\left(-\frac{\lambda}{5}+3(k+1)\right)$.
\end{claim}

\begin{proof}
The proof is by reverse induction on $j$. 
For the base case ($j=m$) observe that $W_m^{(k)}\in[0,1]$, and hence if $\lambda>1$ then $\Pr[W_m^{(k)}>\lambda]=0$. 
Also observe that if $\lambda\leq0$ then $\Pr[W^{(k)}_m>\lambda]\leq1\leq\exp(-\frac{\lambda}{5} + 3(k+1))$.
Assume therefore that $0<\lambda\leq1$. Then,
\begin{align*}
\Pr[W^{(k)}_m>\lambda]&=\Pr[\1\{X_m=1\}\cdot\gamma_m\cdot Z_m^{(k)}>\lambda]\\
&\leq\Pr[X_m=1]\\
&=p_m\\
&\leq\frac{1}{2}\\
&<\exp(-1/5)\\
&\leq\exp(-\lambda/5+3(k+1)).
\end{align*}
This completes the analysis for the base case $(j=m)$.
Now suppose that for any adversary's strategy, for every $\lambda\in\R$, and for every $k\geq0$ it holds that $\Pr[W_{j+1}^{(k)}>\lambda]\leq\exp\left(-\frac{\lambda}{5}+3(k+1)\right)$. We will show the same for $j$. Fix $k\geq0$ and fix $\lambda\in\R$. 
We may assume that $\lambda>0$, as otherwise $\Pr[W^{(k)}_j>\lambda]\leq1\leq\exp(-\frac{\lambda}{5}+3(k+1))$. 
Suppose that the adversary chooses $p_j,q_j$ in round $j$. Then, 
\begin{align}
\Pr[W^{(k)}_j>\lambda]&=\Pr[\1\{X_j=1\}\cdot\gamma_j\cdot Z^{(k)}_j + W^{(k)}_{j+1}>\lambda]\nonumber\\
&=\Pr[X_j=0]\cdot\Pr[\1\{X_j=1\}\cdot\gamma_j\cdot Z^{(k)}_j + W^{(k)}_{j+1}>\lambda | X_j=0]\nonumber\\
&\qquad + \Pr[X_j=1]\cdot\Pr[\1\{X_j=1\}\cdot\gamma_j\cdot Z^{(k)}_j + W^{(k)}_{j+1}>\lambda | X_j=1]\nonumber\\
&\qquad + \Pr[X_j=2]\cdot\Pr[\1\{X_j=1\}\cdot\gamma_j\cdot Z^{(k)}_j + W^{(k)}_{j+1}>\lambda | X_j=2]\nonumber\\
&\leq\Pr[X_j=0]\cdot\Pr[ W^{(k)}_{j+1}>\lambda ]\nonumber\\
&\qquad + \Pr[X_j=1]\cdot\Pr[ W^{(k)}_{j+1}>\lambda-\gamma_j ]\nonumber\\
&\qquad + \Pr[X_j=2]\cdot\Pr[ W^{(k-\gamma_j)}_{j+1}>\lambda ]\nonumber\\
&\leq\left(1-p_j-q_j\right)\cdot\exp\left(-\frac{\lambda}{5}+3(k+1)\right)\nonumber\\
&\qquad + p_j\cdot\exp\left(-\frac{\lambda-\gamma_j}{5}+3(k+1)\right)\nonumber\\
&\qquad + q_j\cdot\exp\left(-\frac{\lambda}{5}+3(k-\gamma_j+1)\right)\label{eq:game1}
\end{align}
\begin{align}
&=\left(1-p_j-q_j\right)\cdot\exp\left(-\frac{\lambda}{5}+3(k+1)\right)\nonumber\\
&\qquad + p_j\cdot e^{\gamma_j/5}\cdot\exp\left(-\frac{\lambda}{5}+3(k+1)\right)\nonumber\\
&\qquad + q_j\cdot e^{-3\gamma_j}\cdot\exp\left(-\frac{\lambda}{5}+3(k+1)\right)\nonumber\\
&=\left(1+p_j\left(e^{\gamma_j/5}-1\right)+q_j\left( e^{-3\gamma_j}-1\right)\right)\cdot\exp\left(-\frac{\lambda}{5}+3(k+1)\right)\nonumber\\
&\leq\left(1+p_j\left(e^{\gamma_j/5}-1\right)+\frac{p_j}{4}\left( e^{-3\gamma_j}-1\right)\right)\cdot\exp\left(-\frac{\lambda}{5}+3(k+1)\right)\nonumber\\
&=\left(1+p_j\left(e^{\gamma_j/5} +  \frac{e^{-3\gamma_j}}{4} - \frac{5}{4} \right)\right)\cdot\exp\left(-\frac{\lambda}{5}+3(k+1)\right)\nonumber\\
&\leq\exp\left(-\frac{\lambda}{5}+3(k+1)\right).\label{eq:game2}
\end{align}
Inequality (\ref{eq:game2}) follows from the fact that $e^{\gamma/5}+\frac{e^{-3\gamma}}{4}-\frac{5}{4}\leq0$ for every $\gamma\in[0,1]$. Inequality (\ref{eq:game1}) follows from the inductive assumption.\footnote{
For the edge case where $-1\leq k-\gamma_j < 0$, in which we cannot apply the inductive assumption, observe that $W^{(k-\gamma_j)}_{j+1}\equiv0$ and hence $\Pr[ W^{(k-\gamma_j)}_{j+1}>\lambda ]=0$.}
\end{proof}

We next show how Lemma~\ref{lem:event} follows from Claim~\ref{claim:game}. Recall that Lemma~\ref{lem:event} states that $\Pr[E_3]\geq1-3\delta$, where $E_3=E_1\wedge E_2$. We analyze $E_1$ and $E_2$ separately.

\begin{claim}\label{claim:E1}
$\Pr[E_1]\geq1-\delta.$
\end{claim}

\begin{proof}
To map algorithm \texttt{ThresholdMonitor} to the setting of Claim~\ref{claim:game}, recall that we 
denote $\Delta=\frac{1}{\eps}\log\left(\frac{1}{\delta}\right)\log\left(\frac{1}{\eps}\log\frac{1}{\delta}\right)$, and
consider an execution of \texttt{ThresholdMonitor}. In every round $i$ let $\gamma_i=f_i(x')$ and define a random variable $X_i$ as follows:
\begin{enumerate}
	\item If $a_i=\top$ then $X_i=2$
	\item If $a_i=\bot$ and $f_i(S_i)+w_i\geq t - 2\Delta$ then $X_i=1$
	\item Otherwise $X_i=0$
\end{enumerate}
To apply Claim~\ref{claim:game} we need to show that for every $i$
$$
\Pr[X_i=1]\leq\frac{1}{2} \qquad\text{and}\qquad \Pr[X_i=2]\geq\frac{1}{4}\cdot\Pr[X_i=1].
$$
Indeed, 
\begin{align*}
\Pr[X_i=1]&=\Pr\left[a_i=\bot \text{ and } f_i(S_i)+w_i\geq t - 2\Delta\right]\\
&=\Pr\left[f_i(S_i)+w_i+\overline{v}_i< t \text{ and } f_i(S_i)+w_i\geq t - 2\Delta\right]\\
&=\Pr\left[t-f_i(S_i)-2\Delta \leq w_i< t - f_i(S_i) - \overline{v}_i\right]\\
&\leq\Pr[v_i<-\Delta]+\Pr\left[t-f_i(S_i)-2\Delta \leq w_i< t - f_i(S_i) - \overline{v}_i | v_i\geq-\Delta\right]\\
&\leq\frac{1}{4}+\Pr\left[t-f_i(S_i)-2\Delta \leq w_i< t - f_i(S_i) +\Delta \right]\\
&\leq\frac{1}{4}+\frac{1}{4}=\frac{1}{2},
\end{align*}
where the last inequality follows from the fact that $w_i$ is sampled from $\Lap(10\Delta)$.
Also,
\begin{align*}
\Pr[X_i=2]&\geq\Pr[f_i(S_i)+w_i\geq t - 2\Delta]\cdot \Pr[X_i=2 | f_i(S_i)+w_i\geq t - 2\Delta]\\
&\geq\Pr[X_i=1]\cdot \Pr[X_i=2 | f_i(S_i)+w_i\geq t - 2\Delta]\\
&=\Pr[X_i=1]\cdot \Pr[a_i=\top | f_i(S_i)+w_i\geq t - 2\Delta]\\
&=\Pr[X_i=1]\cdot \Pr[f_i(S_i)+w_i+\overline{v}_i\geq t | f_i(S_i)+w_i\geq t - 2\Delta]\\
&\geq\Pr[X_i=1]\cdot \Pr[v_i\geq-\Delta]\cdot\Pr[f_i(S_i)+w_i+\overline{v}_i\geq t | f_i(S_i)+w_i\geq t - 2\Delta \text{ and } v_i\geq-\Delta]\\
&\geq\Pr[X_i=1]\cdot \Pr[v_i\geq-\Delta]\cdot\Pr[w_i\geq t - f_i(S_i) + \Delta | w_i\geq t - f_i(S_i) - 2\Delta]\\
&\geq\Pr[X_i=1]\cdot \frac{1}{2}\cdot\frac{1}{2}=\Pr[X_i=1]\cdot\frac{1}{4},
\end{align*}
where the last inequality follows again from the fact that $w_i$ is sampled from $\Lap(10\Delta)$.
Now, let $i^*$ denote the first time step in which $k\leq\sum_{i=1}^{i^*}\1\{X_i=2\}\cdot\gamma_i=\sum_{i=1}^{i^*}\1\{a_i=\top\}\cdot f_i(x')$. Also let 
$$
W = \sum_{i=1}^{i^*} \1\{X_i=1\}\cdot\gamma_i = \sum_{i=1}^{i^*} \1\{i \text{ is an almost-top}\}\cdot f_i(x'). 
$$
By Claim~\ref{claim:game} we have that
$$
\Pr[\;\overline{E_1}\;]\leq\Pr\left[W\geq 5\cdot \log\frac{1}{\delta} + 15(k+1) \right]\leq\delta.
$$
\end{proof}

\begin{claim}\label{claim:E2}
$\Pr[E_2]\geq1-2\delta.$
\end{claim}

\begin{proof}
Consider an execution of \texttt{ThresholdMonitor} on a database $S$, and fix a possible outcome vector $\vec{a}=(a_1,a_2,\dots,a_m)$. Recall that $\vec{a}$ determines all of the queries $f_{\vec{a},i}$ throughout the execution. Note, however, that fixing $\vec{a}$ does not completely determine the noisy values $w_i$ throughout the execution. (Rather, for every query $i$ it only determines whether $\hat{f}_i<t$ or not.) Let $F_1,F_2,\dots,F_m\in\R$ be random variables denoting the values of $w_1,w_2,\dots,w_m$, respectively. Observe that, conditioned on $\vec{a}$, we have that $F_1,F_2,\dots,F_m$ are independent. In addition, for every fixture $I$ of the time steps in which an almost-top occurs, we have that $F_1,F_2,\dots,F_m$ are conditionally independent given $\vec{a}$ and $I$.

Now, let $G_1,G_2,\dots,G_m\in\{0,1\}$ be random variables defined as $G_i=1$ if a special-almost-top occurs in time $i$. As each $G_i$ is a function (only) of $F_i$, we have that $G_1,G_2,\dots,G_m$ are also conditionally independent given $\vec{a}$ and $I$. 
Observe that $\E\left[\left.\sum_i G_i \right| \vec{a},I  \right]\leq\sum_{i\in I} f_{\vec{a},i}(x')$, because every $w_i$ is sampled from $\Lap(10\Delta)$, and hence, even conditioned on $i$ being an almost-top, the probability that $w_i$ falls inside an interval of length $f_{\vec{a},i}(x')$ is at most $f_{\vec{a},i}(x')$. 
Therefore,
\begin{align}
\Pr[E_2] & = \Pr\left[\sum_{i=1}^m G_i \leq 30(k+1)+10\log\frac{1}{\delta} \right]\nonumber\\
& = \sum_{\vec{a},I} \Pr[\vec{a},I] \cdot \Pr\left[\left.  \sum_{i=1}^m G_i \leq 30(k+1)+10\log\frac{1}{\delta} \right| \vec{a},I \right]\nonumber\\
& \geq \sum_{\vec{a},I} \Pr[\vec{a},I] \cdot \Pr\left[\left.  E_1 \wedge  \sum_{i=1}^m G_i \leq 30(k+1)+10\log\frac{1}{\delta} \right| \vec{a},I \right]\nonumber\\
& = \sum_{\vec{a},I} \Pr[\vec{a},I] \cdot \Pr\left[\left.  E_1 \wedge  \sum_{i=1}^m G_i \leq 2\cdot\max\left\{ \sum\nolimits_{i\in I} f_{\vec{a},i}(x') ,\; 15(k+1)+5\log\frac{1}{\delta}  \right\} \right| \vec{a},I \right]\label{eq:explain6}\\
& = \sum_{\vec{a},I} \Pr[\vec{a},I] \cdot \Pr\left[\left.  E_1 \wedge  \sum_{i=1}^m G_i \leq 2\cdot\max\left\{ \E\left[\left. \sum G_i \right|\vec{a},I\right] ,\; 15(k+1)+5\log\frac{1}{\delta}  \right\} \right| \vec{a},I \right]\label{eq:explain7}\\
& \geq \sum_{\vec{a},I} \Pr[\vec{a},I] \cdot \Pr\left[\left.  \sum_{i=1}^m G_i \leq 2\cdot\max\left\{ \E\left[\left. \sum G_i \right|\vec{a},I\right] ,\; 15(k+1)+5\log\frac{1}{\delta}  \right\} \right| \vec{a},I \right]-\delta\label{eq:explain8}\\
& \geq \sum_{\vec{a},I} \Pr[\vec{a},I] \cdot (1-\delta)-\delta\label{eq:explain9}\\
& = 1-2\delta.\nonumber
\end{align}
Equality~(\ref{eq:explain6}) holds since whenever $E_1$ occurs we have that $\sum\nolimits_{i\in I} f_{\vec{a},i}(x') \leq 15(k+1)+5\log\frac{1}{\delta}$. Equality~(\ref{eq:explain7}) holds since $\E\left[\left.\sum_i G_i \right| \vec{a},I  \right]\leq\sum_{i\in I} f_{\vec{a},i}(x')$. Inequality~(\ref{eq:explain8}) holds since $\Pr[E_1]\geq1-\delta$. Inequality~(\ref{eq:explain9}) follows from the Chernoff bound.
\end{proof}

Lemma~\ref{lem:event} now follows by combining Claim~\ref{claim:E1} and Claim~\ref{claim:E2}.

\section{Proof of Theorem~\ref{thm:kprivacy}}\label{sec:ProofOfComposition}

In this section we prove Theorem~\ref{thm:kprivacy}. 
We begin with a short overview of the proof (the formal proof is given afterwards).
 
\paragraph{Proof overview.}
Fix two neighboring databases $S,S'$. We want to show that the outcome distributions of $\texttt{ThresholdMonitor}(S)$ and $\texttt{ThresholdMonitor}(S')$ are similar. To that end, we ``break'' algorithm \texttt{ThresholdMonitor} into several ``smaller'' components, analyze the privacy guarantees of each component, and then argue about the overall privacy guarantees of \texttt{ThresholdMonitor} using composition theorems. Technically, the definition of these components depends on $S$ and $S'$, but this still allows the argument to go through. These components should not be thought of as standalone algorithms (they only exist as part of the proof of Theorem~\ref{thm:kprivacy}).

\medskip 

Before presenting the formal proof, we state the necessary preliminaries regarding adaptive composition for differential privacy~\cite{DRV10}. Let $\eps,\delta$ be parameters. For $b\in\{0,1\}$, define the following $\ell$-round game against an adversary $A$, denoted as $\texttt{Experiment}_{\ell,\eps,\delta}^{b}(A)$.

\begin{center}
\noindent\fboxother{
\parbox{.9\columnwidth}{
{\bf ${\boldsymbol{\texttt{Experiment}_{\ell,\eps,\delta}^{b}(A)}}$}
\begin{enumerate}[topsep=5pt]
	\item For $i=1,\dots,\ell$:
	\begin{enumerate}
		\item The adversary $A$ outputs two distributions $\DDD_i^0$ and $\DDD_i^1$ such that $\DDD_i^0\approx_{(\eps,\delta)}\DDD_i^1$.
		\item The adversary receives a sample $y_i\sim\DDD_i^b$.
	\end{enumerate}
	\item Output $(y_1,\dots,y_{\ell})$ and the internal randomness of the adversary $A$.
	
\end{enumerate}
}}
\end{center}

In every round of this experiment, the adversary specifies two (similar) distributions $\DDD_i^0$ and $\DDD_i^1$ and gets a sample from $\DDD_i^b$. Intuitively, the adversary's goal is to guess the bit $b$. However, since it is restricted to choose distributions $\DDD_i^0$ and $\DDD_i^1$ such that $\DDD_i^0\approx_{(\eps,\delta)}\DDD_i^1$, then its ability to guess the bit $b$ is limited. Formally, the following theorem states that $A$'s view of the experiment (and in particular, $A$'s guess for $b$) is distributed roughly the same for both values of $b$.

\begin{theorem}[\cite{DRV10}]\label{thm:DRV10}
For every $\eps,\delta,\hat{\delta}\geq0$, every $\ell\in\N$, and every adversary $A$ it holds that 
$$\texttt{Experiment}_{\ell,\eps,\delta}^{0}(A)\approx_{(\hat{\eps},\ell\delta+\hat{\delta})}\texttt{Experiment}_{\ell,\eps,\delta}^{1}(A),$$
where $\hat{\eps}=\sqrt{2\ell\ln\left(\frac{1}{\hat{\delta}}\right)}\cdot\eps+\ell\eps(e^\eps-1)$.
\end{theorem}

We are now ready to present the proof of Theorem~\ref{thm:kprivacy}.

\begin{proof}[Proof of Theorem~\ref{thm:kprivacy}]
Fix two neighboring databases $S$ and $S'=S\cup\{x'\}$ for some $x'\in X$. We partition the execution of \texttt{ThresholdMonitor} into $\ell=k/\log\frac{1}{\delta}$ epochs based on the value of the counter $c(x')$, where $x'$ is the data element that appears in $S'$ but not in $S$. Informally, during each epoch, the value of $c(x')$ lies within an interval of length $\log\frac{1}{\delta}$, which ensures that the outcome distribution of each epoch is similar when executing on $S$ or on $S'$ (as in Section~\ref{sec:ThresholdMonitor}). Theorem~\ref{thm:kprivacy} then follows from Theorem~\ref{thm:DRV10}, since \texttt{ThresholdMonitor} can be viewed as the composition of these epochs.

\begin{algorithm*}[t!]
\caption{\bf $
\boldsymbol{\texttt{TM}_{x'}}
$}\label{alg:TMSS}

{\bf Input:} Database $D\in X^*$, privacy parameters $\eps,\delta$, threshold $t$, parameter $k$, counter values $c:X\rightarrow\R$, 
and an adaptively chosen stream of counting queries $f_i:X\rightarrow[0,1]$.

\begin{enumerate}[leftmargin=15pt,rightmargin=10pt,itemsep=1pt,topsep=3pt]

\item Set $c(x')\leftarrow\max\left\{ c(x') ,\; k-\log\frac{1}{\delta} \right\}$. Delete from $D$ every element $x$ such that $c(x)\geq k$.

\item Denote $\Delta=\frac{1}{\eps}\log\left(\frac{1}{\delta}\right)\log\left(\frac{1}{\eps}\log\frac{1}{\delta}\right)$.

\item In each round $i$, when receiving a query $f_i$, do the following:

	\begin{enumerate}[leftmargin=15pt,itemsep=1.5pt,topsep=1pt]
		
		\item Let $w_i\leftarrow\Lap\left(10\Delta\right)$ and let $v_i\leftarrow\Lap(\frac{1}{\eps}\log\frac{1}{\delta})$.
		
		\item Denote $\overline{v}_i=\min\left\{v_i,\; \Delta \right\}$.
		
		\item Let $\hat{f}_i = f_i(D) + w_i + \overline{v}_i$.

		\item If $\hat{f}_i< t$, then output $a_i=\bot$.
		
		\item Otherwise:
		
		\begin{itemize}[leftmargin=20pt,itemsep=1.5pt,topsep=1pt]
			\item Output $a_i=\top$.
			\item For every $x\in X$ set $c(x)\leftarrow c(x)+f_i(x)$.
			\item Delete from $D$ every element $x$ such that $c(x)\geq k$.
		\end{itemize}
		
		\item Proceed to the next iteration.

\end{enumerate}	
\end{enumerate}
\end{algorithm*}

We now give the formal details. Consider algorithm $\texttt{TM}_{x'}$ (presented in Algorithm~\ref{alg:TMSS}). This algorithm is similar to algorithm \texttt{ThresholdMonitor}, except that the counters $c(\cdot)$ are initialized differently. (As we mentioned, this algorithm should not be thought of as a standalone algorithm; it is only part of the analysis.)
Observe that in $\texttt{TM}_{x'}$ we initialize $c(x')\geq k-\log\frac{1}{\delta}$, and that $x'$ is deleted from the data once $c(x')\geq k$. Hence, an identical analysis to that of Section~\ref{sec:ThresholdMonitor} shows that for every neighboring databases that differ on $x'$, say $D$ and $D'=D\cup\{x'\}$, we have that the outcome distribution of $\texttt{TM}_{x'}(D)$ and of $\texttt{TM}_{x'}(D')$ are similar. Specifically, for every adversary $\widetilde{A}$ interacting with $\texttt{TM}_{x'}$ we have that 
$$\widetilde{\BBB}(D)\approx_{\left(O(\eps),3\delta\right)}\widetilde{\BBB}(D'),$$
where $\widetilde{\BBB}$ denotes the function that simulates the interaction and returns the transcript (cf. algorithm~\ref{alg:adaptivealg}). 
Fix an adversary $A$ that interacts with \texttt{ThresholdMonitor}, and let $\BB:X^n\rightarrow\{\bot,\top\}^*$ denote the function that simulates the interaction. We need to show that
$$
\BBB(S)\approx_{(\eps_0,3k\delta)}\BBB(S').
$$ 

We now reformulate algorithm $\BBB$ (ie., the interaction between \texttt{ThresholdMonitor} and the adversary $A$) as $\ell=k/\log\frac{1}{\delta}$ applications of algorithm $\texttt{TM}_{x'}$. Specifically, consider algorithm \texttt{TailoredComposition}, defined below.

\begin{center}
\noindent\fboxother{
\parbox{.9\columnwidth}{
{\bf Algorithm \texttt{TailoredComposition}}\\ 
Input: Database $T$ (where $T$ is either $S$ or $S'$)
\begin{enumerate}[topsep=5pt]
	\item Let $G=1$, and instantiate a counter $\overline{c}(x)=0$ for every $x\in X$.
	\item Instantiate algorithm $\texttt{TM}_{x'}$ on the database $T$ with the counters $\overline{c}$.
	\item For $i=1,2,\dots$ do
	\begin{enumerate}
		\item Obtain the next query $f_i$ from the adversary $A$. 
		\item Give $f_i$ to algorithm $\texttt{TM}_{x'}$ and obtain an answer $a_i$.
		\item Give $a_i$ to the adversary $A$.
		\item If $a_i=\top$ set $\overline{c}(x)\leftarrow\overline{c}+f_i(x)$ for every $x\in X$.
		\item If $\overline{c}(x')\geq G\cdot\log\frac{1}{\delta}$ and $\overline{c}(x')< k$ then 
		\begin{itemize}
			\item Set $G\leftarrow G+1$.
			\item End the current execution of $\texttt{TM}_{x'}$.
			\item Re-instantiate $\texttt{TM}_{x'}(T)$ with the counters $\overline{c}$.
		\end{itemize}
	\end{enumerate}
	\item When $A$ halts, output $z=(q_1,a_1,q_2,a_2,\dots)$.
\end{enumerate}
}}
\end{center}

Now observe that \texttt{TailoredComposition} is identical to $\BBB$, that is, to the function that simulates \texttt{ThresholdMonitor} (with parameter $k$) interacting with the adversary $A$. Specifically, algorithm \texttt{TailoredComposition} keeps track of the counters $\overline{c}$ in a way that is identical to the values of these counters in an execution of \texttt{ThresholdMonitor}, and these counters are given to algorithm $\texttt{TM}_{x'}(T)$ when it is re-executed. 
 Therefore, in order to show that $\BBB(S)\approx_{(\eps_0,3k\delta)}\BBB(S')$ it suffices to show that
$$
\texttt{TailoredComposition}(S) \approx_{(\eps_0,3k\delta)} \texttt{TailoredComposition}(S').
$$
This follows by mapping algorithm \texttt{TailoredComposition} to the settings of Theorem~\ref{thm:DRV10} (the composition theorem).
Specifically, 
for $b\in\{0,1\}$ and $g\in\{1,2,\dots,\ell\}$, let $z_g^b$ denote the portion of the output vector $(q_1,a_1,q_2,a_2,\dots)$ that corresponds to time steps in which $G=g$ during the execution of \texttt{TailoredComposition} on $S_b$, where $S_0=S$ and $S_1=S'$. 
Also, for $g\in\{1,2,\dots,\ell\}$ let $Z_g^b$ be a random variable representing the value of $z_g^b$. Every such $Z_g^b$ is the outcome of the interaction between the adversary and a single execution of algorithm $\texttt{TM}_{x'}(S_b)$. Therefore, for every $g\in\{1,2,\dots,\ell\}$ we have that $Z_g^0\approx_{(O(\eps),3\delta)} Z_g^1$. Thus, by Theorem~\ref{thm:DRV10} we get that
$$\texttt{TailoredComposition}(S)=(Z_1^0,\dots,Z_{\ell}^0)\approx_{(\eps_0,3k\delta)}(Z_1^1,\dots,Z_{\ell}^1)=\texttt{TailoredComposition}(S'),$$
for $\eps_0=O\left(\sqrt{k}\eps+k\eps^2\right)$.

\end{proof}

\end{document}